\documentclass[final,12pt,notheorems]{colt2023} 

\usepackage{etoolbox}
\newtoggle{generic}
\togglefalse{generic}

\iftoggle{generic}
{
\usepackage[utf8]{inputenc}
\usepackage{geometry}
\usepackage{natbib}
}{
}

\usepackage{paralist}

\newcommand{\mylongtitle}{On Classification-Calibration of Gamma-Phi Losses}

\iftoggle{generic}{
  \title{\mylongtitle}
}
{
  \title[\mylongtitle]{\mylongtitle}
  \usepackage{times}
}

\usepackage{tikz}

\usepackage{enumitem}
\setlist[enumerate]{label*=\arabic*.}

\definecolor{cbblue}{HTML}{CAE6FF}
\definecolor{cbyellow}{HTML}{FFEDBA}
\definecolor{fgblue}{HTML}{1E88E5}
\definecolor{fgyellow}{HTML}{E4427D}

\usepackage{algorithm}
\usepackage{algpseudocode}
\algrenewcommand\algorithmicrequire{\textbf{Input:}}
\algrenewcommand\algorithmicensure{\textbf{Output:}}

\usepackage{graphicx}

\usepackage{listings}
\usepackage{xcolor}

\definecolor{codegreen}{rgb}{0,0.6,0}
\definecolor{codegray}{rgb}{0.5,0.5,0.5}
\definecolor{codepurple}{rgb}{0.58,0,0.82}
\definecolor{backcolour}{rgb}{0.95,0.95,0.92}

\lstdefinestyle{mystyle}{
    backgroundcolor=\color{backcolour},
    commentstyle=\color{codegreen},
    keywordstyle=\color{magenta},
    numberstyle=\tiny\color{codegray},
    stringstyle=\color{codepurple},
    basicstyle=\ttfamily\footnotesize,
    breakatwhitespace=false,
    breaklines=true,
    captionpos=b,
    keepspaces=true,
    numbers=left,
    numbersep=5pt,
    showspaces=false,
    showstringspaces=false,
    showtabs=false,
    tabsize=2
}

\usepackage{amssymb}
\usepackage{amsmath}
\usepackage{amsthm}
\usepackage{bm}
\usepackage{bbold}

\usepackage{varioref}
\usepackage{hyperref} 
\usepackage{cleveref}

\iftoggle{generic}{
  \newtheorem{theorem}{Theorem}[section]
\newtheorem{lemma}[theorem]{Lemma}
\crefname{lemma}{lemma}{lemmas}
\Crefname{lemma}{Lemma}{Lemmas}

\Crefname{figure}{Figure}{Figures}
\crefname{proposition}{Prop.}{Prop.}
\Crefname{proposition}{Proposition}{Propositions}

\newtheorem{corollary}[theorem]{Corollary}
\newtheorem{proposition}[theorem]{Proposition}
\newtheorem{conjecture}[theorem]{Conjecture}

\theoremstyle{definition}
\newtheorem{example}[theorem]{Example}
\newtheorem{definition}[theorem]{Definition}
\newtheorem{assumption}[theorem]{Assumption}
\Crefname{assumption}{Assumption}{Assumptions}
\theoremstyle{remark}
\newtheorem{remark}[theorem]{Remark}

}
{

\newtheorem{theorem}{Theorem}[section]
\newtheorem{lemma}[theorem]{Lemma}
\crefname{lemma}{lemma}{lemmas}
\Crefname{lemma}{Lemma}{Lemmas}

\Crefname{figure}{Figure}{Figures}
\crefname{proposition}{Prop.}{Prop.}
\Crefname{proposition}{Proposition}{Propositions}

\newtheorem{corollary}[theorem]{Corollary}
\newtheorem{proposition}[theorem]{Proposition}

\theoremstyle{definition}
\newtheorem{example}[theorem]{Example}
\newtheorem{definition}[theorem]{Definition}

\Crefname{assumption}{Assumption}{Assumptions}
\theoremstyle{remark}
\newtheorem{remark}[theorem]{Remark}

}

\DeclareMathOperator*{\argmax}{arg\,max}

\newcommand\vzero{\bm{0}}

\newcommand\tsp{\tau}

\newcommand{\bfv}{\mathbf{v}}

\newcommand{\bfw}{\mathbf{w}}
\newcommand{\bfp}{\mathbf{p}}
\newcommand{\bfq}{\mathbf{q}}
\newcommand{\calL}{\mathcal{L}}
\newcommand{\vt}{v^{t}} 
\newcommand{\bfvt}{\mathbf{v}^{t}} 
\newcommand{\bfwt}{\mathbf{w}^{t}} 
\newcommand{\wt}{w^{t}} 
\newcommand{\tvt}{\tilde{v}^{t}} 
\newcommand{\bftvt}{\tilde{\mathbf{v}}^{t}} 
\newcommand{\sigmat}{\sigma^t} 

\newcommand{\barR}{\overline{\mathbb{R}}} 

\newcommand{\GammaSI}{\texttt{Gamma-SI}}
\newcommand{\GammaPD}{\texttt{Gamma-PD}}
\newcommand{\PhiNDZ}{\texttt{Phi-NDZ}}

\newcommand{\myproperty}[1]{\texttt{P}\textsubscript{#1}}

\iftoggle{generic}
{
  \author{
    Yutong Wang, Clayton Scott\\
    \texttt{\{yutongw, clayscot\}@umich.edu} \\
    Electrical Engineering and Computer Science \\
    University of Michigan \\
    Ann Arbor, MI 48109
  }
}{
\coltauthor{\Name{Yutong Wang} \Email{yutongw@umich.edu}\\
 \Name{Clayton Scott} \Email{clayscot@umich.edu}\\
 \addr{University of Michigan}}
}

\begin{document}

\maketitle

\begin{abstract}
Gamma-Phi losses constitute a family of multiclass classification loss
functions that generalize
the logistic and other common losses, and have found application in the
boosting literature. We
establish the first general sufficient condition for the
classification-calibration (CC) of such losses. To our knowledge, this
sufficient condition gives the first family of nonconvex multiclass
surrogate losses for which CC has been fully justified. In addition, we
show that a previously proposed sufficient condition is in fact not
sufficient. This contribution highlights a technical issue that is
important in the study of multiclass CC but has been neglected in prior
work.
\end{abstract}

\iftoggle{generic}{}
{
}
 %

\begin{sloppypar}

\section{Introduction}

Multiclass classification into $k \ge 2$ categories is one of the most commonly encountered tasks in machine learning.
To formulate the task mathematically, labelled training instances $\{(x_{i},y_{i})\}_{i = 1}^{n}$ are drawn from a joint distribution $P$ over $\mathcal{X} \times [k]$ where $[k] := \{1,\dots,k\}$ and $\mathcal{X}$ is a  space of unlabelled instances.
The goal is to select a \emph{classifier} $h : \mathcal{X} \to [k]$ that makes the fewest mistakes, i.e.,
 the \emph{01-risk}
  \(R_{01,P}(h) := \mathbb{E}_{(X,Y) \sim P} \left[ \mathbb{I}\{Y \ne h(X)\}\right]\) should be as low as possible.
Here, $\mathbb{I}$ denotes the indicator function.
A fundamental strategy for learning a classifier is \emph{empirical risk minimization} (ERM), which selects an \(h\) minimizing the empirical 01-risk
  \(\hat{R}_{01}^{n}(h) := \frac{1}{n}\sum_{i=1}^{n} \mathbb{I}\{y_{i} \ne h(x_{i})\}\) over a class of functions, e.g., histogram classifiers.

Directly minimizing the empirical 01-risk is difficult due to the discrete nature of the objective.
To address this, it is common to employ continuous-valued \emph{class-score functions} $f = (f_{1},\dots, f_{k}): \mathcal{X} \to \mathbb{R}^{k}$,
whose components represent preference for the respective class,
and are used in lieu of the discrete-valued \(h\).
The classifier associated to $f$ is defined as
$\argmax \circ f(x) :=\argmax_{j=1,\dots,k} f_{j}(x)$, where ties are broken arbitrarily.
For selecting \(f\), a \emph{surrogate loss function} $\calL: [k] \times \mathbb{R}^{k} \to \mathbb{R}_{\ge 0}$ is employed. The quantity $\calL(y, f(x))$ is
the loss incurred by $f$ when instance $x$ has label $y$.
The \emph{$\calL$-risk} is defined to be
\(  R_{\calL,P}(f):=\mathbb{E}_{(X,Y) \sim P} [\calL(Y,f(X))]\),
and $f$ is selected by minimizing the \emph{empirical} $\calL$-risk \(\hat{R}_{\calL}^{n}(f)\)  over a class of functions, e.g., neural networks.

It is desirable for the surrogate loss $\calL$ to have the \emph{consistency transfer property}:
Let $P$ be a distribution over $\mathcal{X} \times [k]$.
Suppose that $\{\hat{f}^{(n)}\}_{n}$ is an arbitrary sequence of score functions, e.g., $\hat{f}^{(n)}$ is an empirical $\calL$-risk minimizer, over a class of functions that may depend on \(n\).
Whenever $R_{\calL}(\hat{f}^{(n)}) \to \inf_{f} R_{\calL}(f)$ as $n\to \infty$, we have
$R_{01}(\argmax \circ \hat{f}^{(n)}) \to \inf_{h}R_{01}(h)$ as $n\to \infty$.
  The infimums are, respectively, over all measurable functions $f: \mathcal{X} \to \mathbb{R}^{k}$ and $h : \mathcal{X} \to [k]$.
  
This property justifies \(\calL\)-risk minimization when minimizing the 01-risk is the actual target of interest.
Establishing sufficient conditions for the property is a central goal of the theory of \emph{classification-calibration}  (CC) \citep{zhang2004statistical,tewari2007consistency,duchi2018multiclass}.

The Gamma-Phi losses are a family of losses that have been successfully applied to the design and analysis of multiclass boosting algorithms \citep{beijbom2014guess,saberian2019multiclass}.
Analysis of boosting algorithms have also been approached from the CC theory point of view \citep{bartlett2006adaboost,mukherjee2013theory}.
However, no sufficient condition for the classification calibration of Gamma-Phi losses has been proposed that broadly encompasses a wide range of practical losses.
This work aims to close this gap.
\subsection{Our contributions}

Informally stated, our contributions are:
\begin{compactenum}
  \item
Theorem~\ref{theorem:Gamma-Phi-loss-is-ISC}:
For \(\gamma\) with strictly positive derivative,  and non-increasing \(\phi\) with negative derivative at zero, the associated Gamma-Phi loss \(\calL\) is classification-calibrated.
\item
Theorem~\ref{proposition:counter-example}:
There exists strictly increasing \(\gamma\), and  non-increasing \(\phi\) with negative derivative at zero whose associated Gamma-Phi loss
is \emph{not} classification-calibrated.
\end{compactenum}
The positive result of this work,
Theorem~\ref{theorem:Gamma-Phi-loss-is-ISC},
establishes the first broadly applicable sufficient condition for classification calibration of Gamma-Phi loss.
The negative result of this work,
Theorem~\ref{proposition:counter-example}, shows that the condition that \(\gamma\) has strictly positive derivative cannot be significantly weakened, and that the sufficient condition conjectured by \citet{pires2016multiclass} turns out to be \emph{insufficient}.

\textcolor{black}{The main impact of our theory is that for the first time it establishes CC of \emph{nonconvex} multiclass loss functions, which have been shown to enjoy robustness to label noise \citep{masnadi2008design,amid2019robust} and adversarial contamination \citep{awasthi2021calibration}. For instance, our work establishes CC of the multiclass sigmoid loss ($\gamma = $ identity, $\phi = $ sigmoid). A second impact is that our work calls attention to the importance of considering the boundary of the probability simplex, not just the interior, when proving CC. Several prior works omit this case \citep{zhang2004statistical, zhang2009coherence,amid2019robust}, and thus their proofs are incomplete.
}

Finally, our result provides insight on choosing Gamma-Phi losses that are suitable for multiclass classification.
For instance, Gamma-Phi losses, which have been successful in (offline) multiclass boosting \cite{saberian2019multiclass}, may be useful for  the online regime as well \citep{raman2022online}. See \citet{jung2017online} and the discussion in \S 4.1 therein.

\subsection{Related works}

Gamma-Phi losses were introduced and studied in a series of papers
and have been shown to perform well in boosting
\citep{saberian2019multiclass,saberian2011multiclass,beijbom2014guess}.
Progress towards proving classification-calibration have been made for special instances of Gamma-Phi, namely for the \emph{coherence loss}~\citep{zhang2009coherence} and the \emph{pairwise-comparison} loss~\citep{zhang2004statistical}\footnote{See Remark~\ref{remark:where-zhang-left-off}.
  Our sufficient condition (Theorem~\ref{theorem:Gamma-Phi-loss-is-ISC}) completes and subsumes these partial works.}.

\noindent\textbf{Non-convex multiclass loss functions}.
Many existing sufficient conditions for CC require the components of \(\calL\) to be convex, e.g., in \citet{tewari2007consistency,bartlett2006convexity}.
Gamma-Phi losses in general are non-convex and thus our result
Theorem~\ref{theorem:Gamma-Phi-loss-is-ISC}, which does not require convexity, is complementary to these works.
Non-convex loss have recently received attention in the context of robust learning \citep{huber2011robust}, and
learning with noisy labels \citep{amid2019robust}.

\noindent\textbf{Beyond classification}.
While our work focuses on classification, many works have developed theory for other learning tasks.
\citet{steinwart2007compare} introduced the extension of loss calibration-theory to cost-sensitive classification, regression and unsupervised learning tasks such as density estimation.
\citet{ramaswamy2016convex} developed theory for learning with general discrete losses such as abstention \citep{ramaswamy2018consistent}.
\citet{finocchiaro2019embedding} showed that there exists polyhedral losses that are calibrated with respect to arbitrary discrete losses.

\noindent\textbf{Restricted class of score functions}.
The key result of CC theory relating \(\calL\)-risk and 01-risk minimization assumes working with a sufficiently rich class of score function\textcolor{black}{, i.e., when  $f: \mathcal{X} \to \mathbb{R}^{k}$ range over all measure functions}. Without this assumption, classification-calibration theory is of limited use \citep[\S 9.1]{mukherjee2013theory}.
\textcolor{black}{
To address this,  \cite{duchi2018multiclass}
introduces the notion of \emph{universal-equivalence of loss functions} and extends the key result of CC theory for certain restricted families of \(f\)'s that are ``quantized''.
Moreover, \cite{long2013consistency,zhang2020bayes,awasthi2022multi} study 
\emph{\(\mathcal{H}\)-consistency} for the situation when \(f\) is restricted to some given family \(\mathcal{H}\). \cite{awasthi2022multi} derives generalized regret bounds based on the so-called \emph{minimizability gap}.
}

\subsection{Notations}

Denote by  $k \ge 2$ the number of classes and by
$\Delta^k = \{ \bfp \in \mathbb{R}_{\ge 0}^k: \sum_{j=1}^k p_j= 1\}$
the $k$-probability simplex.
Let $\Delta^k_{\mathtt{desc}} = \{ \bfp \in \Delta^k : p_1 \ge \cdots \ge p_k\}$ denote the set of probability vectors whose entries are non-increasing (descending) with respect to the index.

\noindent\textbf{Operations on vectors}. Let the square bracket with subscript $[\cdot]_j$ be the projection of a vector onto its $j$-th component, i.e., $[\bfv]_j := v_j$ where $\bfv = (v_1,\dots, v_k) \in \mathbb{R}^k$.
Given two vectors $\bfw, \bfv \in \mathbb{R}^{k}$, we write $\bfw \ge \bfv$ if $w_{j } \ge v_{j}$ for all $j \in [k]$.
Likewise,  we write $\bfw > \bfv$ if $w_{j } > v_{j}$ {for all} $j \in [k]$.

\noindent\textbf{Permutations}.
A bijection from $[k]$ to itself is called a \emph{permutation on $[k]$}.
Denote by $\mathtt{Sym}(k)$ the set of all permutations on $[k]$.
We often write $\sigma \sigma'$ instead of $\sigma \circ \sigma'$ for the compositions of two permutations $\sigma, \sigma' \in \mathtt{Sym}(k)$.
For $i,j \in [k]$, let $\tsp_{(i,j)} \in \mathtt{Sym}(k)$ denote the \emph{transposition} which swaps $i$ and $j$, leaving all other elements unchanged.
More precisely, $\tsp_{(i,j)}(i) = j$, $\tsp_{(i,j)}(j) = i$ and $\tsp_{(i,j)}(y) = y$ for $y \in [k] \setminus \{i,j\}$.

\noindent\textbf{Permutation matrices}.
For each $\sigma \in \mathtt{Sym}(k)$, let $\mathbf{S}_{\sigma}$ denote the permutation matrix corresponding to $\sigma$.
In other words, if $\bfv \in \mathbb{R}^{k}$ is a vector, then $[\mathbf{S}_{\sigma} \bfv]_{j} = [ \bfv ]_{\sigma(j)}= v_{\sigma(j)}$.
Below, we abuse notation and simply write \(\sigma(\bfv)\) instead of \(\mathbf{S}_{\sigma}(\bfv)\) when there is no confusion.
Note that if $\sigma, \sigma' \in \mathtt{Sym}(k)$, then $\mathbf{S}_{\sigma \sigma'} = \mathbf{S}_{\sigma} \mathbf{S}_{\sigma'}$.

\section{Background}
In this section, we review the definitions of Gamma-Phi losses and classification-calibration. In the introduction, a multiclass loss is denoted as $\calL: [k] \times \mathbb{R}^{k} \to \mathbb{R}_{\ge 0}$. However, hereinafter, we will use a slightly modified, but mathematically equivalent notation that is more convenient.
\begin{definition}\label{definition:loss-functions}
A \emph{$k$-ary multiclass loss function}
$\calL = (\calL_{1},\dots, \calL_{k}) : \mathbb{R}^k \to \mathbb{R}^k_+$
  is a vector-valued function \textcolor{black}{such that for all \(\bfv \in \mathbb{R}^{k}\) and all \(y,j \in [k]\), we have \(v_y \ge v_j\) implies that \(\calL_y(\bfv) \le \calL_j(\bfv)\)}.
  Given the score vector \(\bfv \in \mathbb{R}^{k}\), the value \(\calL_{y}(\bfv)\) is the loss incurred when the true label is \(y \in [k]\).
  We say that $\calL$ is
    \emph{permutation equivariant}\footnote{
    This property is sometimes referred to as \emph{symmetric} in the literature. \textcolor{black}{However, a symmetric function, say $f$, has the \emph{in}variance property, i.e., \(f(\mathbf{S}_\sigma (\bfv)) = f(\bfv)\), which is different than the notion of \emph{equi}variance as used in \Cref{definition:loss-functions}}. See \citet[\S 3.1]{bronstein2021geometric}, for an 
    in-depth discussion on invariance versus equivariance.}
        if
        $\calL(\mathbf{S}_{\sigma}(\bfv)) = \mathbf{S}_{\sigma}(\calL(\bfv))$
for all $\bfv \in \mathbb{R}^k$ and $\sigma \in \mathtt{Sym}(k)$.
In other words, the classes are viewed symmetrically from the loss function's perspective.
\end{definition}

\begin{definition}
  [Gamma-Phi losses]\label{definition:Gamma-Phi}
  Let $\gamma: \mathbb{R}_{\ge 0} \to \mathbb{R}_{\ge 0}$ be \textcolor{black}{non-decreasing}
  and $\phi: \mathbb{R} \to \mathbb{R}_{\ge 0}$ \textcolor{black}{non-increasing} functions.
 The \emph{Gamma-Phi} loss associated to $\gamma$ and $\phi$ is the loss $\calL \equiv \calL^{\gamma,\phi}$ whose $y$-th component is given by
  \(
    \calL_y(\bfv) := \gamma\Big(\textstyle\sum_{y' \in [k]: y' \ne y} \phi(v_y - v_{y'})\Big).
  \)

%

\end{definition}

  \begin{example}
  When $\gamma(x) := \log(1+x)$ and $\phi(x) := \exp(-x)$, we recover the logistic/cross entropy loss.
  When $\gamma(x) := T\log(1+x)$ and $\phi(x) := \exp((1-x)/T)$ where $T>0$ is a hyperparameter, we recover the \emph{coherence loss}
  used in the
  multiclass \texttt{GentleBoost.C} algorithm~\citep{zhang2009coherence}.
  When $\gamma$ is the identity,  the Gamma-Phi loss reduces to the \emph{pairwise comparison loss}~\citep[Section 4.1]{zhang2004statistical}.
  The multiclass exponential loss, used in \texttt{AdaBoost.MM} \citep{mukherjee2013theory},
  is the pairwise comparison loss when \(\phi(x) := \exp(-x)\).
  When \(\gamma(x) := (x/(1+x))^2\) and \(\phi(x) := \exp(-x)\), we recover the savage loss \citep{saberian2019multiclass}.
  \end{example}

 We now review fundamental definitions and key result of the theory of classification-calibration.

\begin{definition}\label{definition:perm-condition-risk}
  Let \(\bfp \in \Delta^{k}\).
  The \emph{conditional risk} of $\calL$ at \(\bfp\) is the function $C^{\calL}_{\bfp}: \mathbb{R}^k \to \mathbb{R}$ defined by
  \(    C^{\calL}_{\bfp}(\mathbf{v}) = \sum_{y \in [k]} p_y \calL_y(\mathbf{v}).
\)
  The \emph{conditional Bayes risk} is defined as
  $C_{\bfp}^{\calL, *} := \inf_{\bfv \in \mathbb{R}^{k}}C^{\calL}_{\bfp}(\bfv)$.
  When there is no ambiguity about the loss function, we drop the superscript $\calL$ and simply write $ C_{\bfp}(\mathbf{v})$ and $C_{\bfp}^{*}$.
\end{definition}
This ``conditional'' terminology was used in \citet{bartlett2006convexity}.
It was also called \emph{inner $\mathcal{L}$-risk} by \citet{steinwart2007compare}.
The following is from \citet[Definition 1]{zhang2004statistical}.
\begin{definition}\label{definition:ISC}
  A loss $\calL$ is \emph{classification-calibrated}
  if for all $\bfp \in \Delta^k$  and $y$ such that $p_y < \max_j p_j$, we have
  $
  C_{\bfp}^{\calL, *} < \inf \left\{ C_{\bfp}^{\calL}(\bfv): \bfv \in \mathbb{R}^k, \, v_{y} = \max \bfv\right\}
  $.
\end{definition}

Intuitively, the classification-calibration property states that when \(y\) is not the most probable class label, then outputing a score vector \(\bfv\) maximized at \(y\) leads to sub-optimal conditional \(\calL\)-risk.
 Next, we recall the definitions of
 the \emph{01-risk}
  \(R_{01,P}(h) := \mathbb{E}_{(X,Y) \sim P} \left[ \mathbb{I}\{Y \ne h(X)\}\right]\) 
  and the
\emph{$\calL$-risk} 
\(  R_{\calL,P}(f):=\mathbb{E}_{(X,Y) \sim P} [\calL(Y,f(X))]\).
Finally, the key result in classification-calibration theory is

\begin{theorem}[\cite{zhang2004statistical}]\label{theorem: multiclass classification calibration}
  Let $\calL: \mathbb{R}^k \to \mathbb{R}^k_+$ be a permutation equivariant loss function. Let $\mathcal{F}$ be the set of \textcolor{black}{all} Borel functions $\mathcal{X} \to \mathbb{R}^k$.
  If
$\calL$ is classification-calibrated,
  then \(\calL\) has the \emph{consistency transfer property}:
  For all  sequence of function classes $\{\mathcal{F}_n\}_n$ such that $\mathcal{F}_n \subseteq \mathcal{F}$, $\bigcup_n \mathcal{F}_n = \mathcal{F}$, all $\hat f_n \in \mathcal{F}_n$ and all probability distributions $P$ on \(\mathcal{X}\times [k]\)
  \[\textstyle
    R_{\calL,P}(\hat f_n) \overset{P}{\to}  \inf_{f} R_{\calL,P}(f)
    \quad \mbox{implies} \quad
    R_{01,P}({\smash{\argmax}} \circ\hat f_n) \overset{P}{\to} 
    \inf_{h}
    R_{01,P}(h)
  \]
  where the infimums are taken over all Borel functions \(f : \mathcal{X} \to \mathbb{R}^k\) and \(h : \mathcal{X} \to [k]\), respectively.
\end{theorem}
\textcolor{black}{In applications, \(\hat f_n\) is often taken to be an \(\calL\)-risk empirical minimizer over a training dataset of cardinality \(n\). However, the above property holds for \emph{any} sequence of functions $\hat f_n \in \mathcal{F}_n$.
}

\begin{remark}
  \citet{zhang2004statistical} refers to  Definition~\ref{definition:ISC} as \emph{infinity-sample consistency} (ISC), while
  later work by \citet{tewari2007consistency} considers a slightly different definition of \emph{classification-calibration} (CC).
  Moreover, \citet[Theorem 2]{tewari2007consistency} shows that CC characterizes the consistency transfer property,
  while \citet[Theorem 3]{zhang2004statistical} only shows that ISC is sufficient.
  In fact, ISC is also necessary.
  While this fact is simple, it has not been explictly stated to the best of our knowledge.
  Therefore, we include its proof in Section~\ref{sec-appendix:equivalence} of the Appendix.
  Throughout this work, we will use the name ``classification-calibration''.
\end{remark}


\section{Main results}\label{section:Gamma-Phi}
In this section, we consider the Gamma-Phi loss as in Definition~\ref{definition:Gamma-Phi}.

\begin{definition}\label{definition:Gamma-conditions}
  Let $\gamma : \mathbb{R}_{\ge 0} \to \mathbb{R}_{\ge 0}$ be a function satisfying
  $\sup_{x \in [0,\infty)} \gamma(x) = +\infty$.
  We say that $\gamma$ satisfies {condition}\footnote{``SI'', ``PD'' and ``NDZ'' stand for \underline{s}trictly \underline{i}ncreasing,  \underline{p}ositive \underline{d}erivative and
    \underline{n}egative \underline{d}erivative at \underline{z}ero, respectively. }
  (\GammaSI) if $\gamma$ is strictly increasing, i.e., $\gamma(x) < \gamma(\tilde{x})$ if $x < \tilde{x}$, and
  {condition (\GammaPD)}  if $\gamma$ is continuously differentiable and $\frac{d\gamma}{dx}(x) >0$ for all $x \ge 0$.
\end{definition}
  Note that condition (\GammaPD) implies condition (\GammaSI), but the converse is not true.

\begin{definition}\label{definition:Phi-condition}
  Let $\phi : \mathbb{R} \to \mathbb{R}_{\ge 0}$ be a function with the property that
  $\inf_{x \in \mathbb{R}} \phi(x) = 0$.
  We say that $\phi$ satisfies {condition (\PhiNDZ)}  if
  $\phi$ is differentiable, $\frac{d\phi}{dx}(x) \le 0$ for all $x \in \mathbb{R}$, and $\frac{d\phi}{dx}(0) < 0$.
\end{definition}

\begin{theorem}\label{theorem:Gamma-Phi-loss-is-ISC}
  Let $\calL \equiv \calL^{\gamma,\phi}$ be the Gamma-Phi loss
  where
  $\gamma$ satisfies \GammaPD,
  and $\phi$ satisfies \PhiNDZ.
  Then $\calL$ is classification-calibrated.
\end{theorem}
In light of Theorem~\ref{theorem: multiclass classification calibration},
if $\calL$ satisfies the conditions of Theorem~\ref{theorem:Gamma-Phi-loss-is-ISC}, then
\textcolor{black}{$\calL$ satisfies the transfer consistency property}.
  As stated in the introduction, Theorem~\ref{theorem:Gamma-Phi-loss-is-ISC} establishes the first sufficient condition of CC for Gamma-Phi loss.

  \begin{remark}\label{remark:where-zhang-left-off}
  Both the coherence loss \citep{zhang2009coherence} and pairwise comparison loss \citep{zhang2004statistical} where \(\phi\) satisfies (\PhiNDZ)  satisfy the conditions of Theorem~\ref{theorem:Gamma-Phi-loss-is-ISC}.
  On the other hand, the classification-calibration property for these losses has not been established previously due to omissions in the respective proofs.
  In both aforementioned works,
  the proofs
only check the condition in Definition~\ref{definition:ISC} for \(\bfp \in \Delta^{k}\) such that $\bfp > 0$ elementwise.
In this work, we address the case when \(\bfp\) can have zero entries, which requires significant work.
  \end{remark}

  \begin{remark}
    The multiclass savage loss~\citep{saberian2019multiclass} is a Gamma-Phi loss with $\gamma(x) = (x/(1+x))^{2}$ and $\phi(x) = \exp(-2x)$ which does not satisfy the condition of Theorem~\ref{theorem:Gamma-Phi-loss-is-ISC}. More precisely,
the condition
  $\sup_{x \in [0,\infty)} \gamma(x) = +\infty$ fails.
  While the binary savage loss is classification-calibrated~\citep{masnadi2008design},
  to the best of our knowledge it is unknown whether the multiclass savage loss is classification-calibrated.
  \end{remark}
  \color{black}
We present an intuitive proof sketch of Theorem 3.3. The full proof can be found in \Cref{section:proof-positive}.
\begin{proof}[Proof sketch of \Cref{theorem:Gamma-Phi-loss-is-ISC}]
    
    Unwinding the definition, if the loss \(\mathcal{L}\) is not classification-calibrated, then there exists a class-conditional distribution on the labels \(\mathbf{p} \in \Delta^k\), a label \(y \in [k]\), and a sequence \(\{\mathbf{v}^t \in \mathbb{R}^k\}_{t=1,2,\dots}\)  such that \emph{(i)} \(p_y < \max_{j \in [k]} p_j\) is \emph{not} maximal, \emph{(ii)} \(v_y^t = \max_{j \in [k]} v_j^t\) \emph{is} maximal for all \(t\) and \emph{(iii)} the conditional risk
    \(C_{\mathbf{p}}^{\mathcal{L}}(\mathbf{v}^t) \to C_{\mathbf{p}}^{\mathcal{L}, \ast}\) the Bayes conditional risk.

The ``easy'' case is when the sequence \(\{\mathbf{v}^t\}_{t=1,2,\dots}\) has a limit \(\bm{\alpha} \in \mathbb{R}^k\) where \(\alpha_y = \max_{j \in [k]} \alpha_j\) and  \(C_{\mathbf{p}}^{\mathcal{L}}(\bm{\alpha}) = C_{\mathbf{p}}^{\mathcal{L}, \ast}\). Then it is straightforward\footnote{The ``easy'' case was previously addressed in \citet[Theorem 5]{zhang2004statistical}. For the sake of completeness, we include the argument in \Cref{proposition:strict-ordering}.  } to derive a contradiction using the first derivative test for minimality. However, the challenge is that in general the sequence \(\{\mathbf{v}^t \}_{t=1,2,\dots}\) may  diverge\footnote{Note that neither  \citet[Theorem 5]{zhang2004statistical} nor \Cref{proposition:strict-ordering} can be directly applied in the divergent case because doing so would amount to ``taking a derivative at infinity''. Our technique circumvents this by extracting a convergent subvector. It will be interesting to interpret this technique in the  \emph{astral space} formalism \citep{dudik2022convex}.}. 
    
    The major technical contribution of our work is the construction of a modified sequence \(\{\bftvt\}_{t}\), based on \(\{\mathbf{v}^t \}_{t}\), with the following property: there exists some \(\ell \in \{2,\dots, k\}\) such that for all \(t = 1,2,\dots\) the ``subvector'' \((\tvt_1,\dots,\tvt_\ell) = \bm{\alpha} \in \mathbb{R}^\ell\), i.e., the subvector consisting of the first \(\ell\) entries of \(\bftvt\) is constant with respect to \(t\) and equals to \(\bm{\alpha}\).
     Moreover, $C_{\mathbf{q}}^{\mathcal{L}}(\bm{\alpha}) = C_{\mathbf{q}}^{\mathcal{L},*}$, where
      $\mathbf{q} := {(\sum_{j = 1}^\ell p_j)^{-1}}\left( p_1,\dots, {p_\ell}\right)  \in \Delta^\ell_{\mathtt{desc}}$.
    Thus, we have reduced the problem to the ``easy'' case.

    At a high level, our construction proceeds by first showing that the sequence \(\{\mathbf{v}^t\}_t\) converges in the extended-real sense. Then, we proceed with a series of modifications to \(\{\mathbf{v}^t\}_t\) to ensure that there exists an index \( \ell \in [k]\) such that for all \(1\le j \le \ell\) we have \(v_j^t\) is equal to a finite quantity \(\alpha_j \in \mathbb{R}\) for all \(t\), and that for all \(j > \ell\) we have \(\lim_t v_j^t = \pm \infty\) diverges. We show that throughout these modification, the property  \(C_{\mathbf{p}}^{\mathcal{L}}(\mathbf{v}^t) \to C_{\mathbf{p}}^{\mathcal{L}, \ast}\) continues to hold.
\end{proof}

Our analysis highlights the following intuition: if \(\mathcal{L}\) is classification-calibrated for the \(k\)-category problem, then \(\mathcal{L}\) must be classification-calibrated for every \(\ell\)-category subproblem where \(2 \le \ell \le k\).
\color{black}

   Next, we show an example of a Gamma~Phi loss that satisfies the conditions of \citet{pires2016multiclass} and yet is not classification-calibrated.
The paragraph before \citet[Section 3.4.2]{pires2016multiclass} conjectures that
the Gamma-Phi loss is calibrated when $\gamma$ is strictly increasing and $\phi$ satisfies the same condition as in \citet[Theorem 6]{zhang2004statistical}, namely that
$\phi$ is non-negative, non-increasing and $\phi'(0) <0$.
However, in the following example, we give a counterexample to the aforementioned statement.

  \begin{theorem}\label{proposition:counter-example}
 Let $\calL$ be the Gamma-Phi loss   where
  $
    \phi(x) = \exp(-x)
  $
  and
  \[
    \gamma(x) =
    \begin{cases}
      1-(x-1)^2 & : x <1 \\
      2(x-1)^2 +1 & : x \ge 1. \\
    \end{cases}
  \]
  Then \(\gamma\) satisfies (\GammaSI)
  and \(\phi\) satisfies (\PhiNDZ), but
    $\calL$ is not classification-calibrated.
  \end{theorem}

\textcolor{black}{Our proof relies on a careful analysis of the behavior of the loss function \(\calL(\bfv)\) when the argument \(\bfv\) approaches infinity. Moreover, our proof highlights the importance of verifying the main condition in the definition of classification-calibration (\Cref{definition:ISC}) for \(\bfp\) with zero entries.}
The following sections will prove results stated in this section. All omitted proofs of intermediate  results are included in the appendix.

\section{Conditional risks of permutation equivariant losses}\label{section:PERM-losses}

In this section, we study some of the basic properties of the conditional risk (Definition~\ref{definition:perm-condition-risk}) of permutation equivariant losses
(Definition~\ref{definition:loss-functions}).
All omitted proofs can be found in Section~\ref{sec-appendix:PERM-losses} of the Appendix.

\begin{lemma}\label{lemma:equivariance-of-conditional-risk}
  Let $\calL$ be a permutation equivariant loss.
  Let $\sigma \in \mathtt{Sym}(k)$, $\bfv \in \mathbb{R}^k$ and $\bfp \in \Delta^k$ be arbitrary. Then
  \(C_{\bfp}(\bfv) = C_{\sigma(\bfp)}(\sigma(\bfv))\).
  Furthermore, we have $C_{\bfp}^* = C_{\sigma(\bfp)}^*$.
\end{lemma}

\begin{lemma}\label{lemma:transposition-identity}
  Suppose that $\mathcal{L}$ is permutation equivariant.
  Let $\bfp \in \Delta^k$, $y,y' \in [k]$ and $\bfv \in \mathbb{R}^k$.
  Let $\tau \in \mathtt{Sym}(k)$ be the \emph{transposition} of $y$ and $y'$, i.e., $\tau(y) = y'$, $\tau(y') = y$ and $\tau(j) = j$ for all $j \in [k] \setminus \{y,y'\}$.
  Then
    $C_{\bfp}(\bfv) - C_{\bfp}(\tau(\bfv)) =
(p_y-p_{y'}) (\calL_y(\bfv) - \calL_{y'}(\bfv)).$
\end{lemma}

\begin{proposition}
\label{proposition:bubble-sort}
Let $\bfp \in \Delta^k_{\mathtt{desc}}$. Let $\bfv \in \mathbb{R}^k$ be arbitrary. Let $\sigma \in \mathtt{Sym}(k)$ be such that $v_{\sigma(1)} \ge v_{\sigma(2)} \ge \cdots \ge v_{\sigma(k)}$.
Then $C_{\bfp}(\bfv) \ge C_{\bfp}(\sigma(\bfv))$.
\end{proposition}
\begin{proof}
  This proof is essentially Lemma S3.8 from \citet{wang2020weston} Supplemental Materials.
  First, we note that if $\tilde{\sigma} \in \mathtt{Sym}(k)$ is another permutation such that
  $v_{\tilde{\sigma}(1)} \ge v_{\tilde{\sigma}(2)} \ge \cdots \ge v_{\tilde{\sigma}(k)}$, then $\tilde{\sigma}(\bfv) = \sigma(\bfv)$. Thus, it suffices to prove the result while assuming that the permutation $\sigma$ that sorts $\bfv$ is given by the \emph{bubble sort} algorithm:
  \vspace{1em}
  \begin{compactenum}[L1.]
    \item
      Initialize the iteration index $t \gets 0$ and $\bfv^0 := \bfv$,
    \item While there exists $i \in [k]$ such that $\vt_i  < \vt_{i+1}$,
  do
  \begin{compactenum}
  \item Let $\tau^t = \tau_{(i,i+1)} \in \mathtt{Sym}(k)$ be the permutation that swaps $i$ and $i+1$, leaving other indices unchanged.
  \item  $\bfv^{t+1} \gets \tau^t (\bfv^t)$
\item  $t \gets t + 1$
  \end{compactenum}
\item  Output
  $\bfv^{T}$, where $T \gets t$ is the final iteration index.
  \end{compactenum}
  \vspace{1em}

Let $\langle\cdot, \cdot \rangle$ be the ordinary dot product on $\mathbb{R}^k$. Note that $C_{\mathbf{p}}(\bfv) = \langle \mathbf{p}, \mathcal{L}(\bfv)\rangle$.
Furthermore, at termination, there exists $\sigma \in \mathtt{Sym}(k)$ such that $\bfv^T = \sigma(\bfv)$ is sorted as in the statement of Proposition~\ref{proposition:bubble-sort}.
We claim that at every intermediate step $t \in \{0,\dots, T\}$, we have
$\langle \bfp, \mathcal{L}(\bfv^{t})\rangle \ge
\langle \bfp, \mathcal{L}(\bfv^{t+1})\rangle$.
This would prove
Proposition~\ref{proposition:bubble-sort}, since $\langle \bfp, \mathcal{L}(\bfv^0)\rangle = C_{\mathbf{p}}(\bfv)$ and $\langle \bfp, \mathcal{L}(\bfv^T)\rangle = C_{\mathbf{p}}(\sigma(\bfv))$.

Towards proving our claim, let $t$ be an intermediate iteration of the above ``bubble sort'' algorithm, and let $i \in [k]$ be as in L2.
Then
we have
\begin{align*}
  &\langle \bfp, \mathcal{L}(\bfv^t)\rangle -
  \langle \bfp, \mathcal{L}(\bfv^{t+1})\rangle
  \\
  &=
  \langle \bfp, \mathcal{L}(\bfv^t)\rangle -
  \langle \bfp, \mathcal{L}(\tau^t(\bfv^{t}))\rangle
  \qquad \because \mbox{Definition on L2.(b)}
  \\
  &=
  (p_i - p_{i+1})(\mathcal{L}_{i}(\bfvt)- \mathcal{L}_{i+1}(\bfvt))
  \ge 0, \qquad \mbox{Lemma~\ref{lemma:transposition-identity}}
\end{align*}
as desired. \textcolor{black}{The last inequality follows immediately from \Cref{definition:loss-functions} for multiclass losses.
}
\end{proof}

\section{Proof of Theorem~\ref{theorem:Gamma-Phi-loss-is-ISC}}\label{section:proof-positive}
\textcolor{black}{
In this section, we develop the machinery which will be put together at the end of the section to prove Theorem~\ref{theorem:Gamma-Phi-loss-is-ISC}. The first  goal is to prove \Cref{proposition:infimum-attaining-sequences2}, whose proof sketch was introduced above in  \Cref{section:Gamma-Phi}. Roughly speaking, \Cref{proposition:infimum-attaining-sequences2}  derives properties about sequences \(\{\bfvt\}_t\) such that $\lim_t C_{\bfp}(\bfvt) = C_{\bfp}^{\ast}$. Assuming that \(\calL\) is not classification-calibrated, these properties together with \Crefrange{proposition:strict-ordering}{lemma:non-ISC-property} are then used to derive a contradiction, thereby proving Theorem~\ref{theorem:Gamma-Phi-loss-is-ISC}.} 

Now, to prepare for the proof of  \Cref{proposition:infimum-attaining-sequences2}, we state several helper results:

\begin{lemma}
  \label{lemma:constantization3}
  In the situation of Theorem~\ref{theorem:Gamma-Phi-loss-is-ISC}, let $\{\bfvt\}_t\subseteq \mathbb{R}^k$ be a totally convergent sequence and $\bfp \in \Delta^k$. Then $\lim_t C_{\bfp}(\bfvt)$ exists and is $\in [0,+\infty]$.
  If $\{\bftvt\}_t  \subseteq \mathbb{R}^k$ is another totally convergent sequence such that
  $\lim_t \vt_y - \vt_j = \lim_t \tvt_y - \tvt_j$ for all $y,j \in [k]$, then
  $\lim_t C_{\bfp}(\bfvt) = \lim_t C_{\bfp}(\bftvt)$.
\end{lemma}
 The proof of \Cref{lemma:constantization3} can be found in Section \ref{sec-appendix:proof-positive} of the Appendix.
\begin{corollary}
  \label{corollary:constantization2}
In the situation of Theorem~\ref{theorem:Gamma-Phi-loss-is-ISC}, let $\{\bfvt\}_t \subseteq \mathbb{R}^k$ be a totally convergent sequence and
  $S \subseteq [k]$ be a set such that  $\lim_{t} \vt_y \in \mathbb{R}$ for all $y \in S$.
  Define $\{\bftvt\}_t \subseteq \mathbb{R}^k$ by
  $\tvt_j := \vt_j$ if $j \not \in S$ and $\tvt_j := \lim_t \vt_j$ if $j \in S$.
  Then $\lim_t C_p(\bfvt) = \lim_t C_p(\bftvt)$ as elements of $[0,+\infty]$.
\end{corollary}
\begin{proof}
  Note that $\{\bfvt\}_t$ and $\{\bftvt\}_t$ satisfy the conditions of Lemma~\ref{lemma:constantization3}.
\end{proof}

\begin{lemma}\label{proposition:non-negative-infinity}
In the situation of Theorem~\ref{theorem:Gamma-Phi-loss-is-ISC}, let $\calL$ be the Gamma-Phi loss as in Definition~\ref{definition:Gamma-Phi} where $\gamma$ satisfies (\GammaPD) and $\phi$ satisfies (\PhiNDZ).
  Let $\bfp \in \Delta^k$ and $y,y' \in [k]$ be such that $p_{y'} > p_{y}$.
  Suppose $\{\bfvt\}_{t} \subseteq \mathbb{R}^k$ is a sequence where $\liminf_t \vt_{y} - \vt_{y'} > 0$ and $\lim_t C_{\bfp}(\bfvt) < +\infty$ exists.
  Then $\lim_{t}C_{\bfp}(\bfvt) > C^{\ast}_{\bfp}$.
\end{lemma}
\begin{proof}
  Suppose that $\lim_t C_{\bfp}(\bfvt) = C_{\bfp}^*$.
  We show that this leads to a contradiction.
  Since  $p_{y'} > p_y \ge 0$ and $\lim_t C_{\bfp}(\bfvt) < +\infty$,
  we have $\limsup_t p_{y'} \gamma\left(\sum_{j \in [k] \setminus \{y'\}} \phi(\vt_{y'} - \vt_j)\right) < \infty$.
  By our assumption on $\gamma$, we have
  \(M:=\limsup_t \sum_{j \in [k] \setminus \{y'\}} \phi(\vt_{y'} - \vt_j) < \infty.\)

Next, by assumption, there exists $\epsilon >0$ such that $\vt_y \ge  \vt_{y'}+ \epsilon$ for all $t \gg 0$.
Below, we assume $t$ is in this sufficiently large regime.
Hence, for all $j \in [k]$ we have $\vt_y - \vt_j > \vt_{y'} - \vt_j$ and consequently
$\phi(\vt_y - \vt_j) \le \phi(\vt_{y'} - \vt_j)$.
Furthermore,
$\vt_y- \vt_{y'} \ge \epsilon > 0 > -\epsilon \ge \vt_{y'} - \vt_y$ and so
$\phi(\vt_y - \vt_{y'}) \le \phi(\epsilon) < \phi(-\epsilon) \le \phi(\vt_{y'} - \vt_y)$.
Let
$
a^{t} :=
\sum_{j \in [k]\setminus\{ y'\}}
\phi(\vt_{y'} - \vt_j)
$ and $b^{t}:=
\sum_{j \in [k]\setminus \{y\}}
\phi(\vt_y - \vt_j)
$.
Furthermore, define $\tilde a^t := \phi(-\epsilon) + \sum_{j \in [k] \setminus \{y,y'\}} \phi(\vt_{y'} - \vt_j)$
and $\tilde b^t := \phi(\epsilon) + \sum_{j \in [k] \setminus \{y,y'\}} \phi(\vt_y - \vt_j)$.
Observe that \[
  \tilde a^{t}-\tilde b^{t}
  =
  \phi(-\epsilon) - \phi(\epsilon)
+
\sum_{j \in [k] \setminus \{y,y'\}}
\phi(\vt_{y'} - \vt_j)
-
\phi(\vt_y - \vt_j)
  \ge
  \phi(-\epsilon) - \phi(\epsilon).
\]
In summary, we have $0 \le b^t \le \tilde b^t \le \tilde a^t \le a^t \le M < \infty$.
Let $\tau \in \mathtt{Sym}(k)$ be the permutation that swaps $y$ and $y'$.
By
Lemma~\ref{lemma:transposition-identity}, we have
    \[C_p(\bfvt) - C_p(\tau(\bfvt)) =
(p_y-p_{y'}) (\calL_y(\bfvt) - \calL_{y'}(\bfvt))
=
(p_y-p_{y'}) (\gamma(a^t) - \gamma(b^t)).\]
By the Fundamental Theorem of Calculus, we have
\begin{align*}
\gamma(a^t) - \gamma(b^t)
&=
\int_{b^t}^{a^t} \gamma'(x) dx
\ge
\int_{\tilde b^t}^{\tilde a^t} \gamma'(x) dx
\ge
(\tilde a^t - \tilde b^t) \inf_{x \in [\tilde b^t, \tilde a^t]} \gamma'(x)
\\
&\ge
(\phi(-\epsilon) - \phi(\epsilon)) \inf_{x \in [0,M]} \gamma'(x).
\end{align*}
Since $\gamma$ satisfies (\GammaPD), we have \(\gamma'>0\) and so $\delta := \inf_{x \in [0,M]} \gamma'(x) > 0$.
Thus,
\[
\lim_{t \to \infty}
C_{\bfp}(\bfvt)
-C_{\bfp}(\tau(\bfvt))
\ge
(p_y- p_{y'})
(\phi(-\epsilon) - \phi(\epsilon))
\delta
> 0
\]
where the right hand side is a positive quantity independent of $t$.
Therefore, we conclude that
\(\lim_{t \to \infty}
C_{\bfp}(\bfvt)
>
\lim_{t \to \infty}
C_{\bfp}(\tau(\bfvt)),
\)
which is a contradiction of $\lim_{t \to \infty} C_{\bfp}(\bfvt) =  C_{\bfp}^*$.
\end{proof}


Before proceeding, we adopt the notation
$\{\vt\}_{t} \equiv \alpha$ to denote that $\vt = \alpha$ for all $t$, where
$\{\vt\}_{t}\subseteq \mathbb{R}$ is a sequence of real numbers and $\alpha \in \mathbb{R}$ is a constant.
The first part of the following proposition applies under slightly weaker assumptions than that of
Theorem~\ref{theorem:Gamma-Phi-loss-is-ISC}, namely the (\GammaPD) assumption is replaced by the weaker (\GammaSI). The proposition will be used again in
Section~\ref{section:Gamma-Phi-non-calibrated-example}
in the proof of Theorem~\ref{proposition:counter-example}.
\begin{proposition}\label{proposition:infimum-attaining-sequences2}
  \begin{sloppypar}
    Let $\calL$ be a Gamma-Phi loss where $\gamma$ satisfies (\GammaSI) and $\phi$ satisfies (\PhiNDZ).
    Let $\bfp \in \Delta^k_{\mathtt{desc}}$ and $z \in [k]$ be such that
    $C_{\mathbf{p}}^* = \inf \{ C_{\mathbf{p}}(\bfv) : \bfv \in \mathbb{R}^k, \, v_z = \max \bfv\} $.
    Then there exists a sequence $\{\bfvt\}_t \subseteq \mathbb{R}^{k-1}$ satisfying the following properties:
  \end{sloppypar}
  \begin{compactenum}
    \item
      $\lim_{t \to \infty} C_{\mathbf{p}}^\mathcal{L}(\bfvt) = C_{\mathbf{p}}^*$
    \item there exists an index $\ell \in [k]$ and a
      vector
  $
  \bm{\alpha}: = (\alpha_{1},\dots, \alpha_\ell) \in \mathbb{R}^{\ell}
    $
    such that
          for each $j \in \{1,\dots, \ell\}$ we have $\{\vt_j\} \equiv \alpha_j$ and
          $\lim_t \vt_j = -\infty$ for $j > \ell$.
          In addition, $\alpha_1 = 0$.
    \item
      Let
      $\mathbf{q} := {(\sum_{j = 1}^\ell p_j)^{-1}}\left( p_1,\dots, {p_\ell}\right)  \in \Delta^\ell_{\mathtt{desc}}$.
    Then $C_{\mathbf{q}}(\bm{\alpha}) = C_{\mathbf{q}}^*$.
  \end{compactenum}
  Furthermore, suppose
   $z > 1$, $p_{z-1} > p_z$,
  and
$\gamma$ satisfies (\GammaPD).
  Then $\{\bfvt\}$ can be chosen such that
    $\alpha_j = 0$ for all $j \in [z]$ .
\end{proposition}
\begin{proof}
  Let $\{\bfvt\}_t \subseteq \mathbb{R}^{k-1}$ be a sequence such that $\lim_{t \to \infty} C^{\mathcal{L}}_{\mathbf{p}}(\bfvt) = C_{\mathbf{p}}^*$ and $\vt_z = \max \bfvt$ for all $t \in \mathbb{N}$.
  A sequence \(\{\bfvt\}_{t}\) satisfying this preceding condition is said to have property \myproperty{0}.
  Throughout this proof, $t$ denotes the index of the sequence where ``for all $t$'' means ``for all $t \in \mathbb{N}$''.
  We will repeatedly modify the sequence $\bfvt$ until all properties \myproperty{1} to \myproperty{3} below are met in addition to \myproperty{0}.
  Under the assumptions in the ``Furthermore'' part of the Proposition,
  we will continue to modify the sequence  until
  properties \myproperty{4} and \myproperty{5} marked by ``\textasteriskcentered'' are further satisfied.

  \vspace{1em}

  \underline{Properties}:
  \begin{compactenum}[\color{white}1.]
  \item
    \label{step:att-inf-leader-is-zero}
    \texttt{P}\textsubscript{1}.
    $\max \bfvt = 0$ for all $t$
  \item
    \label{step:att-inf-totally-convergent}
    \texttt{P}\textsubscript{2}.
    $\{\bfvt\}_t$ is totally convergent and
  $\{\vt_j\}_t$ has a limit in $[-\infty,0]$ for each $j \in[k]$
  \item\label{step:att-inf-ell}
    \texttt{P}\textsubscript{3}.
  there exists an $\ell \in [k]$  such that for each $j \in [\ell]$, we have $\{\vt_j\} \equiv \alpha_j$ where $\alpha_j \in (-\infty, 0]$ and for each $j \in [k]\setminus [\ell]:= \{\ell+1,\dots, k\}$, we have $\lim_t \vt_j = -\infty$.
  In fact,
  $\ell \in [k]$ is the largest index such that $\lim_t \vt_\ell > -\infty$.
\item
    \label{step:att-inf-top-z-are-zeros}
    \texttt{P}\textsubscript{4}\textasteriskcentered.
  the sequence $\{\vt_j\}_t \equiv 0$
  for each $j \in [z]$
  \item\label{step:att-inf-ell-gte-z}
    \texttt{P}\textsubscript{5}\textasteriskcentered.
   $\ell \ge z$.
  \end{compactenum}

  \vspace{1em}



  \noindent\underline{Properties \ref{step:att-inf-leader-is-zero} and \ref{step:att-inf-totally-convergent}}.
  First, note that $C_p(\bfv) = C_p(\bfv- c \mathbf{1})$ for any $c \in \mathbb{R}$ and any $\bfv \in \mathbb{R}^k$.
  Replacing each $\bfvt$ by $\bfvt - (\max \bfvt) \mathbf{1}$ for all $t$, we may assume $\vt_z = \max \bfvt = 0$ for all $t$.
  In particular, $\vt_j \in (-\infty,0]$ for all $j \in [k]$ and $t$.
  Passing to a subsequence if necessary, we may assume that $\{\bfvt\}_t$ is totally convergent (Lemma~\ref{lemma:every-sequence-has-totally-convergent-subsequence}) whose components have limits in \([-\infty,0]\).

\vspace{1em}

\noindent\underline{Property~\ref{step:att-inf-ell}}.
  Let $\sigmat \in \mathtt{Sym}(k)$ be the permutation that sorts $\bfvt$ in non-increasing order
  as in
Proposition~\ref{proposition:bubble-sort},
  i.e., $\vt_{\sigmat(1)} \ge \cdots \ge \vt_{\sigmat(k)}$.
  By Proposition~\ref{proposition:bubble-sort}, $C_{\bfp}(\sigmat(\bfvt)) \le C_{\bfp}(\bfvt)$ and hence $\lim_t C_{\bfp}(\sigmat(\bfvt)) = C_{\bfp}^*$ as well.
  We now replace $\bfvt$ by $\sigmat(\bfvt)$.
Note that \myproperty{1} is preserved by the sorting, but  \myproperty{2} may no longer hold.
Replacing by a subsequence if necessary, we can that $\bfvt$ is totally convergent (Lemma~\ref{lemma:every-sequence-has-totally-convergent-subsequence}).
Since \myproperty{1} holds,  we have $\lim_t \vt_j \in [-\infty,0]$ and hence \myproperty{2} holds.

By Property~\ref{step:att-inf-totally-convergent}. By the sorting in the preceding paragraph, we have that $\lim_t \vt_1 \ge \cdots \ge \lim_t \vt_k$.
  Now, let $\ell \in [k]$ be the largest index such that $\lim_t \vt_\ell > -\infty$.
  Such an index exists because $\lim_t \vt_1 =\lim_t \max \bfvt =  0$ (due to \myproperty{1}).
  Let $\alpha_j := \lim_t \vt_j \in (-\infty,0]$ for each $j \in \{1,\dots, \ell\}$.
  Define $\bftvt$ such that $\{\tvt_j\}_t \equiv \alpha_j$ for $j \in \{1,\dots, \ell\}$ and $\{\tvt_j\}_t = \{\vt_j\}_t$ for $j > \ell$.
  Then by Corollary~\ref{corollary:constantization2}, we have
  $\bftvt$ is totally convergent,
  and $\lim_t C_{\bfp}(\bftvt) = \lim_t C_{\bfp}(\bfvt)$.
  Replace $\vt$ by $\tvt$.
  Thus, we have that \(\{\bfvt\}_{t}\) satisfies \myproperty{1} to \myproperty{3}.

\vspace{1em}

\noindent\underline{Property \ref{step:att-inf-top-z-are-zeros}}.
  By \myproperty{1}, we already have $\vt_z = \max \bfvt = 0$.
  By the assumption that $p_{z-1} > p_z$  and that $\bfp \in \Delta^k_{\mathtt{desc}}$, we have $p_j > p_z$ for each $j \in [z-1]$.
  Furthermore, by \myproperty{2}, \(\{\bfvt\}_{t}\) is totally convergent.
  Hence, for a fixed \(j \in [z-1]\) by definition $\vt_z - \vt_j = -\vt_j$ has a limit in $[0,\infty]$.
  By Lemma~\ref{proposition:non-negative-infinity}, $\lim_t -\vt_j \not \in (0,\infty]$ and thus $\lim_t -\vt_j = 0$.
  Now, define the sequence $\{\bftvt\}_{t}$ by
  \[
    \tvt_j :=
    \begin{cases}
      0 &: j \in \{1,\dots, z-1\}\\
      \vt_j &: j \in \{z,\dots, k\}
    \end{cases}
  \]
  for all $t$.
  By Corollary~\ref{corollary:constantization2}, $\{\bftvt\}_t$ is also totally convergent and $\lim_t C_{\mathbf{p}}(\bftvt) = \lim_t C_{\mathbf{p}}(\bfvt)$.
  Thus,
  $\lim_t C_{\bfp}(\bftvt) = C_{\bfp}^*$.
  Replacing $\bfvt$ by $\bftvt$, we have that \myproperty{4} holds.
  Clearly, \myproperty{0}, \myproperty{1} and \myproperty{2} all hold.
  Note that \(\lim_{t} \bftvt = \lim_{t} \bfvt\) by construction. Thus, \myproperty{3} still holds.
  Moreover, the \(\ell\) defined in \myproperty{3} is not changed when \(\bfvt\)  is replaced by \(\bftvt\).

  \vspace{1em}

\noindent\underline{Property~\ref{step:att-inf-ell-gte-z}}.
By \myproperty{4}, $\{\vt_j\} \equiv 0$ for each $j \in [z]$. Hence, by the definition of $\ell$ in \myproperty{3}, we have $\ell \ge z$.

\vspace{1em}
  We now proceed with the rest of the proof for Proposition~\ref{proposition:infimum-attaining-sequences2}.
  Consider the sequence $\{\bfvt\}_t$ constructed as above satisfying \myproperty{0} through \myproperty{3}. Then items 1 and 2 of Proposition~\ref{proposition:infimum-attaining-sequences2} are implied by \myproperty{0} and \myproperty{3} respectively. Now, if the assumptions in the ``Furthermore'' part hold, then the conclusion of the ``Furthermore'' part holds by \myproperty{4}.
  It only remains to check item 3 of
Proposition~\ref{proposition:infimum-attaining-sequences2}.
  Below, we write $[k] \setminus [\ell] := \{\ell+1,\dots, k\}$.
  Now, note that
  \begin{align}
    &
    \lim_t
    C_{\mathbf{p}}(
    \bfvt)
      \nonumber
    \\
    &=
    \sum_{y \in [k]} p_y \gamma\Big(\sum_{j \in [k]: j \ne y} \phi(\lim_t \vt_y - \vt_j)\Big)
      \nonumber
    \\
    &=
    \sum_{y \in [\ell]} p_y \gamma\Big(\sum_{j \in [k]: j \ne y} \phi(\lim_t \vt_y - \vt_j)\Big)
    +
    \underbrace{
      \sum_{y \in [k]\setminus [\ell]} p_y \gamma\Big(\sum_{j \in [k]: j \ne y} \phi(\lim_t \vt_y - \vt_j)\Big)
  }_{=: A}
  \label{equation:aux-quantity-A}
    \\
    &=
    (\underbrace{p_1+\cdots + p_\ell}_{=: S})\sum_{y \in [\ell]} q_y \gamma\Big(\sum_{j \in [k]: j \ne y} \phi(\lim_t \vt_y - \vt_j)\Big)
    +
    A.
  \label{equation:aux-quantity-S}
  \end{align}
  \textcolor{black}{Note that \(A\) is defined as a limit. At this point, it is unclear if this limit exists and is \(\in [0,\infty)\). This will become clear after 
  Equation~\eqref{equation:Cp-is-S-times-Cq-plus-A} below.
  }
  But first,
   we focus on the $\lim_t \vt_y - \vt_j$ case by case\footnote{  The case $y \in  [k]\setminus [\ell] ,\, j \in [k]\setminus [\ell]$ is omitted from the expression because  $\lim_t \vt_y - \vt_j$ cannot be simplified further.
}:
\[
\lim_t \vt_y - \vt_j
=
\begin{cases}
  \alpha_y - \alpha_j &: y \in [\ell],\, j \in [\ell]\\
  \alpha_y- \lim_t \vt_j = +\infty &: y \in [\ell],\, j \in [k]\setminus [\ell]\\
  \lim_t \vt_y - \alpha_j = -\infty &: y \in [k]\setminus [\ell] ,\, j \in [\ell].
\end{cases}
  \]
  Now,
  \[
 \phi( \lim_{t} \vt_y - \vt_j)
    =\begin{cases}
      \phi(\alpha_y - \alpha_j) &: j \in [\ell] \\
      \phi(+\infty) = 0 &: j \in \{\ell+1,\dots, k\}.
    \end{cases}
  \]

  Putting it all together, we have
  \begin{align}
    \lim_t C_{\mathbf{p}}(\bfvt)
      \nonumber
    &=
    S\sum_{y \in [\ell]} q_y \gamma(\sum_{j \in [k]: j \ne y} \phi(\lim_t \vt_y - \vt_j))
    +
    A
      \nonumber
    \\
    &=
    S\sum_{y \in [\ell]} q_y \gamma(\sum_{j \in [\ell]: j \ne y} \phi(\alpha_y - \alpha_j) + \sum_{j \in [k]\setminus [\ell]: j \ne y} \phi(+\infty))
    +
    A
      \nonumber
    \\
    &=
    S\sum_{y \in [\ell]} q_y \gamma(\sum_{j \in [\ell]: j \ne y} \phi(\alpha_y - \alpha_j))
    +
    A
      \nonumber
    \\
    &=
    S \cdot C_{\mathbf{q}}(\bm{\alpha})
    +
    A.
    \label{equation:Cp-is-S-times-Cq-plus-A}
  \end{align}
Note that \(C_{\mathbf{p}}^*
     =
    \lim_t C_{\mathbf{p}}(\bfvt) \in [0,\infty)
    \) and \(S \cdot C_{\mathbf{q}}(\bm{\alpha}) \in [0,\infty)\). Thus, the limit and defines \(A\) exists and is \(\in[0,\infty)\).
  Now, let $\bm{\beta} = (\beta_1,\dots, \beta_\ell) \in \mathbb{R}^\ell$ be arbitrary and define a sequence $\{\bfwt\} \subseteq \mathbb{R}^k$ by
  \[
    \wt_j :=
    \begin{cases}
      \beta_j &: j \in [\ell]\\
      \vt_j &: j \in [k]\setminus [\ell].
    \end{cases}
  \]

  Then analogous to Equation~\eqref{equation:aux-quantity-A}
  above, we have the decomposition
  \[
    \lim_t C_{\mathbf{p}}( \bfwt)
    =
    \sum_{y \in [\ell]} p_y \gamma\Big(\sum_{j \in [k]: j \ne y} \phi(\lim_t \wt_y - \wt_j)\Big)
    +
    \underbrace{
      \sum_{y \in [k]\setminus [\ell]} p_y \gamma\Big(\sum_{j \in [k]: j \ne y} \phi(\lim_t \wt_y - \wt_j)\Big)
  }_{=: B}.
  \]
  We claim that $A = B$ and $\lim_t C_{\mathbf{p}}(\wt) = S\cdot C_{\mathbf{q}}(\bm{\beta}) + A$.
  We first prove that $A = B$. To this end, observe that
  \[
\lim_t \wt_y - \wt_j
=
\begin{cases}
  \beta_y - \beta_j &: y \in [\ell],\, j \in [\ell]\\
  \beta_y- \lim_t \vt_j = +\infty &: y \in [\ell],\, j \in [k]\setminus [\ell]\\
  \lim_t \vt_y - \beta_j = -\infty &: y \in [k]\setminus [\ell] ,\, j \in [\ell]\\
  \lim_t \vt_y - \vt_j&: y \in [k]\setminus [\ell] ,\, j \in [k]\setminus [\ell].
\end{cases}
  \]
  In particular, for $y \in [k]\setminus [\ell],j \in [\ell]$, we have
$
\lim_t \wt_y - \wt_j
=-\infty
=
\lim_t \vt_y - \vt_j$.
  Thus,
  \begin{align*}
    B &=
    \sum_{y \in [k]\setminus [\ell]} p_y \gamma\Big(
    \sum_{j \in [\ell]: j \ne y} \phi(\lim_t \wt_y - \wt_j)
    +
    \sum_{j \in [k]\setminus [\ell]: j \ne y} \phi(\lim_t \wt_y - \wt_j)
    \Big)
    \\
      &=
    \sum_{y \in [k]\setminus [\ell]} p_y \gamma\Big(
    \sum_{j \in [\ell]: j \ne y} \phi(-\infty)
    +
    \sum_{j \in [k]\setminus [\ell]: j \ne y} \phi(\lim_t \vt_y - \vt_j)
    \Big)
    \\
      &=
      A.
  \end{align*}
  Next, we have
  \begin{align*}
    \lim_t C_{\mathbf{p}}( \bfwt)
    &=
    \sum_{y \in [\ell]} p_y \gamma\Big(\sum_{j \in [k]: j \ne y} \phi(\lim_t \wt_y - \wt_j)\Big)
    +
    A\\
    &=
    S \sum_{y \in [\ell]} q_y \gamma\Big(\sum_{j \in [\ell]: j \ne y} \phi(\beta_y - \beta_j)
+
\sum_{j \in [k]\setminus [\ell]: j \ne y} \phi(+\infty)\Big)
    +
    A
    \\
    &= S\cdot C_{\mathbf{q}}(\bm{\beta})  + A.
  \end{align*}
  Since $\lim_t  C_{\mathbf{p}}(\bfwt) \ge \lim_t C_{\mathbf{p}}(\bfvt) = C_{\mathbf{p}}^*$,
  we have $C_{\mathbf{q}}(\bm{\beta}) \ge C_{\mathbf{q}}(\bm{\alpha})$.
  Since $\bm{\beta}$ is arbitrary, this proves that $C_{\mathbf{q}}(\bm{\alpha}) = C_{\mathbf{q}}^*$.
\end{proof}
\begin{corollary}\label{lemma:Cp-is-Cq}
  \begin{sloppypar}
    Let $\calL$ be a Gamma-Phi loss  where $\gamma$ satisfies (\GammaSI) and $\phi$ satisfies  (\PhiNDZ).
    Let $\{\bfvt\}_t$ be any sequence satisfying items 1, 2 and 3 of Proposition~\ref{proposition:infimum-attaining-sequences2}.
    If $p_y = 0$ for each $y > \ell$, then $C_{\bfq}(\bm{\alpha}) = \lim_t C_{\bfp}(\bfvt)$.
  \end{sloppypar}
\end{corollary}
\begin{proof}
  In Equation~\eqref{equation:Cp-is-S-times-Cq-plus-A}, we showed that
  $
    \lim_t C_{\mathbf{p}}(\bfvt)
    =
    S \cdot C_{\mathbf{q}}(\bm{\alpha})
    +
    A
    $
    where $S$ and $A$ are defined on Equations~\eqref{equation:aux-quantity-S} and
    \eqref{equation:aux-quantity-A} respectively.
    If $p_y = 0$ for all $y > \ell$, then clearly $S = 1$ and $A = 0$.
\end{proof}

  \begin{lemma}\label{proposition:strict-ordering}
In the situation of Theorem~\ref{theorem:Gamma-Phi-loss-is-ISC}, suppose that $\bfq \in \Delta^\ell$ and $\bm{\alpha} \in \mathbb{R}^\ell$ are such that $\bm{\alpha}$ is a minimizer of
    $C_{\bfq}(\cdot)$ and $\alpha_1 = \alpha_2$. Then $q_1 = q_2$.
  \end{lemma}

  \begin{lemma}
    \label{lemma:non-ISC-property}
  Suppose $\calL$ is not classification-calibrated.
  Then
  there exists a probability vector $\mathbf{p} \in \Delta^k_{\mathtt{desc}}$ and
  an index $z \in \{2,\dots, k\}$  satisfying 1) $p_{z-1} > p_z$ and 2)
  $C_{\mathbf{p}}^* = \inf \{ C_{\mathbf{p}}(\bfv) : \bfv \in \mathbb{R}^k, \, v_z = \max \bfv\} $.
  \end{lemma}

Proofs of both lemmas can be found in Section \ref{sec-appendix:proof-positive} of the Appendix. Now, we conclude with the
  \begin{proof}[Proof of Theorem~\ref{theorem:Gamma-Phi-loss-is-ISC}]
    Let $\bfp \in \Delta^k_{\mathtt{desc}}$ and $z \in \{2,\dots, k\}$ be as in
Lemma~\ref{lemma:non-ISC-property}, which states that $\bfp$ and $z$ satisfy the conditions of
Proposition~\ref{proposition:infimum-attaining-sequences2}.
Next, let $\ell \in [k]$, $\bm{\alpha} \in \mathbb{R}^\ell$, and $\bfq \in \Delta^{\ell}_{\mathtt{desc}}$ be as in Proposition~\ref{proposition:infimum-attaining-sequences2}, which satisfy $C_{\mathbf{q}}(\bm{\alpha}) = C_{\mathbf{q}}^*$ and $q_z < q_{z-1} \le  q_1 = \max \bfq$.
Let $\tau \in \mathtt{Sym}(\ell)$ be the permutation which swaps $z$ and $2$ leaving all elements in $[\ell] \setminus \{2,z\}$ unchanged.
Then
\[
C_{\tau(\bfq)}^* = C_{\bfq}^* = C_{\bfq}(\bm{\alpha}) = C_{\tau(\bfq)}(\tau(\bm{\alpha})).
\]
Let $\tilde{\bfq}:= \tau(\bfq)$ and $\tilde{\bm{\alpha}} := \tau(\bm{\alpha})$.
Then $[\tilde{\bm{\alpha}}]_1 =[\bm{\alpha}]_{\tau(1)} =  \alpha_1 =0$ and $[\tilde{\bm{\alpha}}]_2 =  [\bm{\alpha}]_{\tau(2)} = \alpha_z =0$.
In particular, $\tilde{\alpha}_1 = \tilde{\alpha}_2$.
Thus, by
Lemma~\ref{proposition:strict-ordering}, we have $\tilde{q}_1 = \tilde{q}_2$. However,
$\tilde{q}_1 = [\bfq]_{\tau(1)} = q_1$ and $\tilde{q}_2 = [\bfq]_{\tau(2)} = q_z$. Since $q_z < q_1$, we have a contradiction.
  \end{proof}

  \section{Proof of Theorem~\ref{proposition:counter-example}: A Gamma-Phi loss that is not classification-calibrated}\label{section:Gamma-Phi-non-calibrated-example}

    For $r \in [\tfrac{1}{2},1]$, define
    $\bfp := [r,1-r, 0,\dots,0]  \in \Delta^k_{\mathtt{desc}}$.
    Thus, for $\bfv \in \mathbb{R}^k$, we have
    \begin{equation}
      C_{\bfp}(\bfv) =
      r \gamma\Big(\sum_{j \in [k] \setminus \{1\}} \phi(v_1 - v_j)\Big)
      + (1-r) \gamma\Big(\sum_{j \in [k]\setminus \{2\}} \phi(v_2 - v_j)\Big).
      \label{equation:counter-example-eq1}
    \end{equation}
    Below, fix
\(r \in (\tfrac{1}{2},\tfrac{2}{3}]\) once and for all.
    Denote by $\mathtt{SEQ}$ the set of all sequences
    $\{\bfvt\}_t$ satisfying Proposition~\ref{proposition:infimum-attaining-sequences2} items 1, 2 and 3.
    For each sequence $\{\bfvt\}_t \in \mathtt{SEQ}$, there exists an $\ell \in [k]$ (defined as in Proposition~\ref{proposition:infimum-attaining-sequences2} item 2) such that $\lim_t \vt_j = -\infty$ if and only if $j \in [k]$ satisfies $j > \ell$.
    Below, we choose a particular sequence \(\{\bfvt\}_{t}\):
    \begin{equation}
      \label{equation:minimality-of-ell}
      \mbox{
        Fix a sequence $\{\bfvt\}_t \in \mathtt{SEQ}$ such that $\ell$ is as small as possible.
      }
    \end{equation}
    Furthermore, let $\bm{q} \in \Delta^{\ell}_{\mathtt{desc}}$, $\bm{\alpha} \in \mathbb{R}^\ell$ be from
    Proposition~\ref{proposition:infimum-attaining-sequences2} item 3.
    Recall that  we have $\lim_t C_{\bfp}(\bfvt) = C_{\bfp}^*$ and that $C_{\bfq}(\bm{\alpha}) = C_{\bfq}^*$.
    Furthermore,
    Proposition~\ref{proposition:infimum-attaining-sequences2} asserts that $\alpha_{1}=\vt_1 = 0$.

 We prove
in Lemma~\ref{lemma:counter-example-helper} in the Appendix
    that in fact
      $\ell = 2$ and $\bm{\alpha} = (0,0)$.
      To sketch the main idea here briefly, the key step is showing that
                          \(F(x) := r \gamma\left(\phi( x)\right) + (1-r) \gamma\left(\phi(-x) \right)\)
                          has a unique minimum at \(x = 0\).
                          The \(\gamma\) and \(\phi\) in Theorem~\ref{proposition:counter-example} is chosen specifically to achieve this.

                            Now, define another sequence $\{\bfwt\}_{t= 1}^\infty \subseteq \mathbb{R}^k$ where
                            \(
\bfwt = (0, \tfrac{1}{t}, -t,\dots, -t)
                            \).
                            Then  by construction we have
                            \(\lim_t C_{\bfp} (\bfwt)
                              =
                              C_{\bfq}(\bm{\alpha})
                              =
                              C_{\bfq}^{\ast}
                              = C_{\bfp}^*\)
                        and $\argmax_{j \in [k]} \wt_j = 2$ for all $t$.
                        Thus, we have constructed an example of $\bfp$ and $y \in [k]$ where
                        \(C_{\bfp}^* = \inf \{ C_{\bfp}(\bfv): \bfv \in \mathbb{R}^k, \, v_y = \max \bfv\}\).
                        This shows that $\calL$ is not classification-calibrated (Definition~\ref{definition:ISC}).
Thus, we have proven Theorem~\ref{proposition:counter-example}.
\hfill \qedsymbol{}

\section{Discussion}

\begin{sloppypar}
  In this work, we establish the first sufficient condition for the classification calibration of a Gamma-Phi loss in terms of the functional properties of \(\gamma\) and \(\phi\) in Theorem~\ref{theorem:Gamma-Phi-loss-is-ISC}.
  We also showed that our sufficient condition cannot be significantly weakened (in terms of the condition on $\gamma$) via our counter-example in Proposition~\ref{proposition:counter-example}.
\end{sloppypar}

\begin{sloppypar}
  For future work, studying the connection between non-convex Gamma-Phi losses and \(\mathcal{H}\)-consistency \citep{awasthi2022multi} may be fruitful.
  Another interesting direction is the relationship between non-convex Gamma-Phi losses and learning with label noise \citep{amid2019robust}.
  \textcolor{black}{Finally, an important future direction is to establish a concrete regret/excess risk bound.  While \cite{zhang2004statistical} guarantees the existence of such a risk bound, the proof is not constructive. Deriving concrete bounds is likely to require additional assumptions on \(\gamma\) and \(\phi\).}
\end{sloppypar}
\iftoggle{generic}{
\section*{Acknowledgements}
}{
\acks{The authors were supported in part by the National Science Foundation under awards 1838179
and 2008074, and by the Department of Defense, Defense Threat Reduction Agency under award
HDTRA1-20-2-0002. YW is supported by the Eric and Wendy Schmidt AI in Science Postdoctoral
Fellowship, a Schmidt Futures program.}
}

\bibliography{references}

\newpage

\appendix

\section{The ISC property characterizes consistency transfer property}\label{sec-appendix:equivalence}

\begin{proposition}
  Let $\calL: \mathbb{R}^k \to \mathbb{R}^k_+$ be a multiclass loss function that does not have the ISC property, and $\mathcal{F}$ be the set of Borel functions $\mathcal{X} \to \mathbb{R}^k$.
  Then $\calL$ does not have the consistency transfer property, namely:
  There exists a sequence of functions $\hat f_n \in \mathcal{F}$ and a probability distribution $P$ on \(\mathcal{X}\times [k]\) such that
  \[
  \textstyle
    R_{\calL,P}(\hat f_n) \overset{P}{\to}  \inf_{f} R_{\calL,P}(f)
    \mbox{ holds but }
    R_{01,P}({\smash{\argmax}} \circ\hat f_n) \overset{P}{\to} \inf_{h}
    R_{01,P}(h)
    \mbox{ fails.}
  \]
    \textcolor{black}{Here, the infimums are taken over all Borel functions \(f : \mathcal{X} \to \mathbb{R}^k\) and \(h : \mathcal{X} \to [k]\), respectively.}
\end{proposition}
The above proposition is essentially the multiclass analog of the result \citep[Theorem 1, part 3c \(\implies\) part 3a]{bartlett2006convexity}.
Our proof below is a simple extension  to the multiclass case of the argument in the paragraph before
\S 2.4 in
\citet{bartlett2006convexity}.
\begin{proof}
  Since \(\calL\) does not have the ISC property, there exists \(\bfp \in \Delta^{k}\) and \(y \in [k]\) such that
 $y$ such that $p_y < \max_j p_j$, and
  $
  C_{\bfp}^{\calL, *} = \inf \left\{ C_{\bfp}^{\calL}(\bfv): \bfv \in \mathbb{R}^k, \, v_{y} = \max \bfv\right\}
  $.
  Let \(\{\bfv^{{n}}\}_{n}\) be a sequence such that
  \(v_{y}^{n} = \max \bfv^{n}\)
  and
\(\lim_{n\to\infty} C_{\bfp}^{\calL}(\bfv^{n}) = C_{\bfp}^{\ast}\).
Below, fix some arbitrary \(x \in \mathcal{X}\).
Define \(P\) on \(\mathcal{X} \times [k]\) such that \(P(X=x) = 1\) and \(P(X= x, Y =y') = p_{y'}\) for each \(y' \in [k]\).

Next, we define a pair of  maps \(\vec{V}:\mathbb{R}^{k} \to \mathcal{F}\) and \(\vec{F}: \mathcal{F} \to \mathbb{R}^{k}\) as follows.
Given \(f \in \mathcal{F}\), let \(\vec{V}(f) := (f_{1}(x),\dots, f_{k}(x)) \in \mathbb{R}^{k}\).
Given \(\bfv \in \mathbb{R}^{k}\) and \(x' \in X\), define \(\vec{F}(\bfv)(x') := \bfv\) to be the constant-valued map with value \(\bfv\).
Since \(P(X=x) = 1\), we have for all \(f \in \mathcal{F}\)
\[C_{\bfp}^{\calL}(\vec{V}(f)) = R_{\calL,P}(f)
\quad \mbox{and} \quad
1 - \max p_{\argmax \vec{V}(f)}= R_{01,P}(\argmax \circ f).\]

Now, for each \(n\), let \(\hat{f}^{n} := \vec{F}(\bfv^{n}) \in \mathcal{F}\).
Then by construction we have
\(\vec{V}(\hat{f}^{n}) = \bfv^{n}\) and
\begin{align*}
\lim_{n\to\infty}R_{\calL,P}(\hat{f}^{n}) &=
                                            \lim_{n\to\infty}C_{\bfp}^{\calL}(\vec{V}(\hat{f}^{n}))
                                            \quad \because
                                            C_{\bfp}^{\calL}(\vec{V}(f)) = R_{\calL,P}(f),\,
                                            \forall f \in \mathcal{F}
  \\&=
                                            \lim_{n\to\infty}C_{\bfp}^{\calL}(\bfv^{n})
  \\&
  = C_{\bfp}^{\calL,\ast}
 \qquad \qquad \because{\mbox{assumption on \(\bfv^{n}\)}}
  \\
&=
\inf_{\bfv \in \mathbb{R}^{k}} C_{\bfp}^{\calL}(\bfv)
 \qquad \because{\mbox{definition}}
       \\
&=
\inf_{f \in \mathcal{F}} C_{\bfp}^{\calL}(\vec{V}(f))
 \quad \because{\mbox{\(\vec{V}\) is surjective onto \(\mathbb{R}^{k}\)}}
       \\
&=
            \inf_{f \in \mathcal{F}}R_{\calL,P}(f)
            \quad \because
                                            C_{\bfp}^{\calL}(\vec{V}(f)) = R_{\calL,P}(f),\,
                                            \forall f \in \mathcal{F}
\end{align*}
On the other hand, since the \(\argmax\) breaks tie arbitrarily \citep[Lemma 4]{tewari2007consistency}, we can choose \(\argmax \vec{V}(\hat{f}^{n}) = \argmax \bfv^{n} = y\).
Therefore, we have
\[R_{01,P}(\argmax \circ \hat{f}^{n})
=
1 - p_{y}
>
1 - \max_{j \in [k]} p_{j}
=
\inf_{h}
    R_{01,P}(h).
\]
Thus, we have constructed a sequence \(\hat f_{n}\) such that
  \[
  \textstyle
    R_{\calL,P}(\hat f_n) \overset{P}{\to}  \inf_{f} R_{\calL,P}(f)
    \mbox{ holds but }
    R_{01,P}({\smash{\argmax}} \circ\hat f_n) \overset{P}{\to} \inf_{h}
    R_{01,P}(h)
    \mbox{ fails.}
  \]
as desired.
\end{proof}

\section{Extended reals}\label{section:Gamma-Phi loss}

In this section, we review results on the extended real numbers. For reference, see \citet[\S 1.16]{oden2017applied}.

\begin{definition} [Convergence in extended reals] \label{definition:totally-delta-convergent}
  Let $\barR := \mathbb{R} \cup \{ \pm \infty\}$ and $\barR_{\ge 0} = \mathbb{R}_{\ge 0} \cup \{+\infty\}$.
  A sequence $\{z^t\}_t \subseteq \mathbb{R}$ has a limit in $\barR$ if one of the following holds: 1) $\{z^t\}$ has a limit in the usual sense, 2) for all $c \in \mathbb{R}$, we have $z^t \ge c$ (resp.\ $z^t \le c$) for all $t$ sufficiently large in which case we say $\lim_t z^t = +\infty$ (resp.\ $\lim_t z^t = -\infty$).
\end{definition}

The following are elementary properties of convergence in the extended reals:
\begin{lemma}\label{prop:basic-property-of-extlim-sum}
  Let $\{z^t\}$ and $\{\tilde{z}^t\}$ be sequences in $\mathbb{R}$ with limits in $\barR$.
  Then
  $z^t + \tilde{z}^t$ has a limit in $\overline{\mathbb{R}}$ equal to $\lim_t z^t + \lim_t \tilde{z}^t$ if any of the following holds:
  \begin{compactenum}
  \item at least one of $\lim_t z^t$ or $\lim_t \tilde{z}^t$ is finite, i.e., $\in \mathbb{R}$,
  \item $\{z^t\}_t$ and $\{\tilde{z}^t\}_t$ are both $\subseteq [0,\infty)$,
  \item $\{z^t\}_t$ and $\{\tilde{z}^t\}_t$ are both $\subseteq (-\infty,0]$.
  \end{compactenum}
\end{lemma}
\begin{proof}
  The results follow respectively from the following standard properties of the extended real numbers: 1.\ \(\pm \infty + c = \pm \infty\) for any \(c \in \mathbb{R}\), 2.\ \(+\infty + \infty = +\infty\) and 3.\ \(-\infty -\infty = -\infty\).
\end{proof}

\begin{definition}
  A function $f : \mathbb{R} \to \mathbb{R}$ is \emph{monotone non-increasing} (resp.\ non-decreasing) if $f(x) \ge f(y)$ for all $x,y \in \mathbb{R}$ such that $x \le y$ (resp.\ $x \ge y$).
\end{definition}
\begin{lemma}
  \label{lemma:monotone-function-extension}
  Let $f : \mathbb{R} \to \mathbb{R}$ be continuous and monotone non-increasing.
  Suppose that $\{z^t\}_t \subseteq \mathbb{R}$ has a limit $z^* \in \barR$.
  Then $f(z^t)$ has a limit $\in \barR$ and
  \begin{equation}
    \label{equation:limits-of-monotone-functions}
    \lim_t f(z^t) = \begin{cases}
      f(z^*) &: z^* \in \mathbb{R} \\
      \inf_{x \in \mathbb{R}} f(x) &: z^* = + \infty\\
      \sup_{x \in \mathbb{R}} f(x) &: z^* = - \infty.
    \end{cases}
  \end{equation}
    Thus, the statement $\lim_t f(z^t) = f(\lim_t z^t)$ is correct.
  When $f$ is monotone non-\emph{decreasing},
    Equation~\eqref{equation:limits-of-monotone-functions} holds with the $\inf$ and $\sup$ swapped.
\end{lemma}
\begin{proof}
  If $z^* \in \mathbb{R}$, then the result is simply the definition of continuity.
  Next, suppose that $z^* = +\infty$. Our goal is to show that $\lim_t f(z^t)$ exists and converges to $I:=\inf_{x \in \mathbb{R}} f(x)$.

  Consider the case  that $I = -\infty$.
  Then for any $U \in \mathbb{R}$, there exists $u \in \mathbb{R}$ such that $f(u) \le U$.
  Since $z^* = +\infty$, $z_t \ge u$ for all $t \gg 0$ sufficiently large, and in which case
  $f(z_t) \le f(u) \le U$.
  Since $U \in \mathbb{R}$ is arbitrary, we have that $\lim_t f(z^t) = -\infty$
  (Definition~\ref{definition:totally-delta-convergent}).

  Now, consider the case  that $I \in \mathbb{R}$.
  Then by definition $f(z^t) \ge I$ for all $t$.
  Furthermore, for any $\epsilon > 0$, there exists $u$ such that $f(u) \le I + \epsilon$.
  Again, since $z^* = +\infty$, $z_t \ge u$ for all $t \gg 0$ sufficiently large, in which case
  $f(z_t) \le f(u) \le I+\epsilon$.
  Since $\epsilon>0$ is arbitrary, this proves that $\lim_t f(z^t) = I$.
  The proof for the case when $z^* = - \infty$ is completely analogous.
  Furthermore, when $f$ is monotone non-decreasing, the roles of $\inf$ and $\sup$ are clearly swapped.
\end{proof}



\begin{definition}
  A sequence of vectors $\{\bfvt\}_t \in \mathbb{R}^k$ is \emph{totally convergent} if for all $y,j \in [k]$, both sequences of real numbers $\{\vt_y\}$  and $\{\vt_y - \vt_j\}$ have limits in $\overline{\mathbb{R}}$.
\end{definition}

\begin{lemma}\label{lemma:every-sequence-has-totally-convergent-subsequence}
  Every sequence $\{\bfvt\}_t \in \mathbb{R}^k$ has a subsequence that is totally convergent.
\end{lemma}
\begin{proof}
  Every sequence of real numbers has a convergent subsequence with limit in $\mathbb{R} \cup \{ \pm \infty\}$.\footnote{This is follows from \(\mathbb{R}\cup \{\pm \infty\}\) being the compactification of \(\mathbb{R}\). See \citet[\S 1.16]{oden2017applied}.}
  By repeatedly passing to convergent subsequences, first for all $j \in [k]$, then for all pairs $j,j' \in [k]$ with $j < j'$, we get the desired result.
\end{proof}

\section{Omitted proofs of results from Section~\ref{section:PERM-losses}}\label{sec-appendix:PERM-losses}
\begin{proof}[Proof of Lemma~\ref{lemma:equivariance-of-conditional-risk}]
To prove that
  \(C_{\bfp}(\bfv) = C_{\sigma(\bfp)}(\sigma(\bfv))\), we note that
  \[
    C_{\bfp}(\bfv)
    =
    \sum_{y \in [k]}
    p_y \calL_y(\bfv)
    =
    \sum_{y \in [k]}
    p_{\sigma(y)} {\calL}_{\sigma(y)}(\bfv)
    =
    \sum_{y \in [k]}
    [\sigma(\bfp)]_{y}
    {\calL}_{y}(\sigma(\bfv))
    =
    C_{\sigma(\bfp)} (\sigma(\bfv)).
  \]
  For the ``Furthermore'' part, note that \(\bfv \mapsto \sigma^{-1}(\bfv)\) is a bijection from \(\mathbb{R}^{k}\) to itself. Hence,
  \begin{align*}
    C_{\bfp}^* &= \inf \{C_{\bfp}(\bfv) : \bfv \in \mathbb{R}^k\} \qquad
 \because                \mbox{Definition~\ref{definition:perm-condition-risk}}\\
    &=
      \inf \{C_{\bfp}(\sigma^{-1}(\bfv)) : \bfv \in \mathbb{R}^k\}
\qquad \because \mbox{\(\bfv \mapsto \sigma^{-1}(\bfv)\) is a bijection}
    \\
    &=
    \inf \{C_{\sigma(\bfp)}(\sigma(\sigma^{-1}(\bfv))) : \bfv \in \mathbb{R}^k\}
      \qquad \because
      C_{\bfp}(\bfv) = C_{\sigma(\bfp)}(\sigma(\bfv))
    \\
    &=
    \inf \{C_{\sigma(\bfp)}(\bfv) : \bfv \in \mathbb{R}^k\}
      \qquad \because
\sigma(\sigma^{-1}(\bfv))
=
\bfv
\end{align*}
  The right hand side is equal to
  $
  \inf \{C_{\sigma(\bfp)}(\bfv) : \bfv \in \mathbb{R}^k\} = C_{\sigma(\bfp)}^*$
  since
  \(
\sigma(\sigma^{-1}(\bfv))
=
\bfv
  \)
\end{proof}
\begin{proof}[Proof of Lemma~\ref{lemma:transposition-identity}]
  Since \(\calL\) is permutation equivariant, we have for all \(j \in [k]\) that
  \begin{equation}
    \calL_{j}(\mathbf{S}_{\tau}(\bfv))
    =
    [\calL(\mathbf{S}_{\tau}(\bfv))]_{j}
    =
    [\mathbf{S}_{\tau}(\calL(\bfv))]_{j}
    =
    [\calL(\bfv)]_{\tau(j)}
    =
    \calL_{\tau(j)}(\bfv).
\label{equation:transposition-identity-eq1}
\end{equation}
To finish the proof, we have
\begin{align*}
    &C_{\bfp}(\bfv) - C_{\bfp}(\tau(\bfv))\\
    &
      =\sum_{j \in [k]} p_j
      \left(\calL_j(\bfv) -\calL_j(\tau(\bfv))\right)
  \\
    &
      =\sum_{j \in [k]}
p_j
      \left(
      \calL_j(\bfv) - \calL_{\tau(j)}(\bfv)
\right      )
      \quad \because \mbox{Equation~\eqref{equation:transposition-identity-eq1}
      }
\\
    &
=(p_y \calL_y(\bfv) + p_{y'}\calL_{y'}(\bfv))
-
(p_y \calL_{y'}(\bfv) + p_{y'} \calL_y(\bfv))
      \quad \because \mbox{\(\tau\) is a transposition}
\\
&=
p_y (\calL_y(\bfv) - \calL_{y'}(\bfv)) + p_{y'} (\calL_{y'}(\bfv) - \calL_{y}(\bfv))
\\
&=
(p_y-p_{y'}) (\calL_y(\bfv) - \calL_{y'}(\bfv)),
\end{align*}
as desired.
\end{proof}

\section{Omitted proof of results from Section~\ref{section:proof-positive}}\label{sec-appendix:proof-positive}

\begin{proof}[Proof of Lemma~\ref{lemma:constantization3}]
  Let \(y,j \in [k]\) be arbitrary.
Since \(\bfvt\) is totally convergent, \(\vt_y - \vt_{j}\) has a limit in \(\mathbb{R}\cup \{\pm\infty\}\).
Next, since \(\phi\) is monotone and non-negative by condition (\PhiNDZ), we have by Lemma~\ref{lemma:monotone-function-extension}
that \(\phi(\vt_y - \vt_{j})\) has a limit in \([0,+\infty]\).
Now, that \(\lim_{t} C_{\bfp}(\bfvt)\) has a limit in \([0,+\infty]\) follows immediately from Lemma~\ref{prop:basic-property-of-extlim-sum}.

Next, define $a^t_y := \sum_{j \in [k]: j \ne y}  \phi(\vt_y - \vt_{j})$
  and $\tilde a^t_y := \sum_{j \in [k]: j \ne y}  \phi(\tvt_y - \tvt_{j})$.
  We proceed stepwise as follows:

  \vspace{1em}

  \begin{compactenum}[Step 1:]
  \item \label{step1:constanization}$\lim_t \phi(\vt_y - \vt_j) = \lim_t \phi(\tvt_y - \tvt_j)$ as elements of $[0,+\infty]$,
  \item \label{step2:constanization}$\lim_t a^t_y = \lim_t \tilde a^t_y$ as elements of $[0,+\infty]$,
    \item\label{step3:constanization} $\lim_t \gamma(a^t_y) = \lim_t \gamma(\tilde{a}^t_y)$
      as elements of $[0,+\infty]$
    \item\label{step4:constanization} $\lim_t \sum_{y \in [k]} p_y \gamma(a^t_y) = \lim_t \sum_{y \in [k]} p_y \gamma(\tilde{a}^t_y)$
  \end{compactenum}

  \vspace{1em}

  Proof of \underline{Step \ref{step1:constanization}}. From Lemma~\ref{lemma:monotone-function-extension} and the fact that $\phi$ is monotone and continuous, we get that
  $\lim_t \phi(\vt_y - \vt_j) = \phi(\lim_t \vt_y - \vt_j)$
  and
  $\lim_t \phi(\tvt_y - \tvt_j) = \phi(\lim_t \tvt_y - \tvt_j)$.
  Note that Lemma~\ref{lemma:monotone-function-extension} also guarantees that these limits exist.
  Non-negativity of the limit values follows from the non-negativity of $\phi$.

  \underline{Step \ref{step2:constanization}}. From Lemma~\ref{prop:basic-property-of-extlim-sum} and the non-negativity of $\phi$, we have
  \[\lim a^t_y = \sum_{j \in [k] \setminus \{y\}} \lim_t \phi(\vt_y - \vt_j)
    =
    \sum_{j \in [k]\setminus \{y\}} \lim_t \phi(\tvt_y - \tvt_j)
    =
    \lim \tilde{a}^t_y
  \]
  where the equality in the middle follows from Step \ref{step1:constanization}.
  Note that Lemma~\ref{prop:basic-property-of-extlim-sum} also guarantees that these limits exist.

  \underline{Step \ref{step3:constanization}}. This follows from Step \ref{step2:constanization}, Lemma~\ref{lemma:monotone-function-extension} and the non-negativity of $\gamma$ on $[0,\infty)$.

  \underline{Step \ref{step4:constanization}}. This follows from Step \ref{step3:constanization}
  and Lemma~\ref{prop:basic-property-of-extlim-sum}.
\end{proof}

\begin{proof}[Proof of Lemma~\ref{proposition:strict-ordering}]
    Throughout this proof, let \(\gamma'(x) := \frac{d \gamma}{dx}(x)\) and
   \(\phi'(x) := \frac{d \phi}{dx}(x)\).
    Recall that
    \[
      C_{\bfq}(\bfv) = \sum_{y \in [\ell]}q_y\gamma\Big( \sum_{j \in [k] \setminus \{y\}} \phi(v_y - v_j)\Big).
    \]
    For each $y$, define
    $
    \Gamma_y(\bfv) := \gamma'\Big( \sum_{j \in [k] \setminus \{y\}} \phi(v_y - v_j)\Big).$
    Thus
    \begin{equation}
      \frac{\partial C_{\bfq}}{\partial v_y}(\bfv)
      =
      \Big(
      q_y
      \Gamma_y(\bfv)
        \sum_{j \in [k]\setminus \{y\}}
      \phi'(v_y - v_j)
      \Big)
      -
      \Big(
        \sum_{j \in [k]\setminus \{y\}}
        q_j
      \Gamma_j(\bfv)
      \phi'(v_j - v_y)
      \Big).
\label{equation:vanishing-of-partial-derivatives}
\end{equation}
    The vanishing of the first two partial derivatives
    $\begin{bmatrix}
      \frac{\partial C_{\bfq}}{\partial v_1}(\bfv)
      &
      \frac{\partial C_{\bfq}}{\partial v_2}(\bfv)
    \end{bmatrix}= 0$ (i.e., Equation~\eqref{equation:vanishing-of-partial-derivatives} where \(y=1,2\))
    can be cast in matrix form equivalently as follows:
    \[
      \begin{bmatrix}
        q_1 \Gamma_1(\bfv) \\
        q_2 \Gamma_2(\bfv) \\
        q_3 \Gamma_3(\bfv) \\
        \vdots\\
        q_k \Gamma_k(\bfv)
      \end{bmatrix}^\top
      \begin{bmatrix}
        \sum_{j \in [k] \setminus \{1\}} \phi'(v_1 - v_j) &
        -\phi'(v_1 - v_2)
        \\
        -\phi'(v_2 - v_1)
                                                                    &
        \sum_{j \in [k] \setminus \{2\}} \phi'(v_2 - v_j)\\
        -\phi'(v_3 - v_1) &
        -\phi'(v_3 - v_2) \\
        \vdots & \vdots\\
        -\phi'(v_k - v_1) &
        -\phi'(v_k - v_2) \\
      \end{bmatrix}=\vzero.
  \]
  The above equation is satisfied at $\bfv = \bm{\alpha}$, which satisfies $\alpha_1 = \alpha_2$ by assumption.
    \[
      \begin{bmatrix}
        q_1 \Gamma_1(\bm{\alpha}) \\
        q_2 \Gamma_2(\bm{\alpha}) \\
        q_3 \Gamma_3(\bm{\alpha}) \\
        \vdots\\
        q_k \Gamma_k(\bm{\alpha})
      \end{bmatrix}^\top
      \begin{bmatrix}
        \sum_{j \in [k] \setminus \{1\}} \phi'(\alpha_1 - \alpha_j) &
        -\phi'(0)
        \\
        -\phi'(0)
                                                                    &
        \sum_{j \in [k] \setminus \{2\}} \phi'(\alpha_2 - \alpha_j)\\
        -\phi'(\alpha_3 - \alpha_1) &
        -\phi'(\alpha_3 - \alpha_1) \\
        \vdots & \vdots\\
        -\phi'(\alpha_k - \alpha_1) &
        -\phi'(\alpha_k - \alpha_1) \\
      \end{bmatrix}=\vzero.
  \]

  Equivalently, we can rearrange the above equation as
    \begin{align*}
      &
      \begin{bmatrix}
        q_1 \Gamma_1(\bm{\alpha}) \\
        q_2 \Gamma_2(\bm{\alpha})
      \end{bmatrix}^\top
      \begin{bmatrix}
        \sum_{j \in [k] \setminus \{1\}} \phi'(\alpha_1 - \alpha_j) &
        -\phi'(0)
        \\
        -\phi'(0)
                                                                    &
        \sum_{j \in [k] \setminus \{2\}} \phi'(\alpha_2 - \alpha_j)
      \end{bmatrix}
      \\
      &=
      \underbrace{
      \begin{bmatrix}
        q_3 \Gamma_3(\bm{\alpha}) \\
        \vdots\\
        q_k \Gamma_k(\bm{\alpha})
      \end{bmatrix}^\top
      \begin{bmatrix}
        \phi'(\alpha_3 - \alpha_1) \\
        \vdots \\
        \phi'(\alpha_k - \alpha_1) \\
      \end{bmatrix}
    }_{=:d}
      \begin{bmatrix}
        1 & 1
      \end{bmatrix}
      =
      d \mathbf{1}^\top
  \end{align*}
  Furthermore, note that
  \begin{align*}
        \sum_{j \in [k] \setminus \{1\}} \phi'(\alpha_1 - \alpha_j)
        &=
        \phi'(\alpha_1 - \alpha_2)
        +
        \sum_{j \in [k] \setminus \{1,2\}} \phi'(\alpha_1 - \alpha_j)
        \\
        &=
        \phi'(0)
        +
        \sum_{j \in [k] \setminus \{1,2\}} \phi'(\alpha_1 - \alpha_j)
        \\
        &=
        \phi'(0)
        +
        \sum_{j \in [k] \setminus \{1,2\}} \phi'(\alpha_2 - \alpha_j)
        \\
        &=
        \sum_{j \in [k] \setminus \{2\}} \phi'(\alpha_2 - \alpha_j).
  \end{align*}
  Likewise,
  $\Gamma_1(\bm{\alpha})
  =
  \gamma'(\phi(0) + \sum_{j \in [k] \setminus \{1,2\}} \phi(v_1 - v_j))
  =\Gamma_2(\bm{\alpha})$.
  Let $a := \phi'(0)$, $b :=
  \sum_{j \in [k] \setminus \{1,2\}} \phi'(\alpha_1 - \alpha_j)$, and $c := \Gamma_1(\bm{\alpha})$. Since $\gamma'(\cdot) > 0$, we have $c > 0$ and so
    \[
      c
      \begin{bmatrix}
        q_1\\
        q_2
      \end{bmatrix}^\top
      \begin{bmatrix}
        a+b &
        -a
        \\
        -a
                                                                    &
                                                                    a+b
      \end{bmatrix}
      =
      d
      \mathbf{1}^\top
      \implies
      \begin{bmatrix}
        a+b &
        -a
        \\
        -a
                                                                    &
                                                                    a+b
      \end{bmatrix}
      \begin{bmatrix}
        q_1\\
        q_2
      \end{bmatrix}
      =
      \frac{d}{c}
      \mathbf{1}.
  \]
Note that since $\phi' \le 0$ and $\phi'(0) \ne 0$, we have $a \in (-\infty,0)$ and $b \in (-\infty,0]$. 
\textcolor{black}{Next, subtract the second equation from the first one in the above linear system, we get
\(
(2a+b)(q_1-q_2) = 0
\). Since \(2a+b<0\), we  have that \(q_1=q_2\).}
  \end{proof}


\begin{proof}[Proof of Lemma~\ref{lemma:non-ISC-property}]
    By Definition~\ref{definition:ISC}, there exists some $\mathbf{q} \in \Delta^k$ and $y \in [k]$ such that $q_y < \max_{j \in [k]} q_j$ and
  \[
    C_{\bfq}^*=
    \inf \{ C_{\bfq}(\mathbf{v}) : \mathbf{v} \in \mathbb{R}^k, \, v_y = \max_{j \in [k]} v_j\}.
  \]
  The above implies that there exists a sequence $\{\bfvt\}_t \subseteq \mathbb{R}^k$ such that $\lim_{t} C_{\bfq}(\bfvt) = C_{\bfq}^*$ and $\vt_y = \max_{j \in [k]} \vt_j$ for all $t$.
  Let ${\sigma} \in \mathtt{Sym}(k)$ be such that ${\sigma}(\bfq) \in \Delta^k_{\mathtt{desc}}$.
  Let $\tilde{y} := {\sigma}^{-1}(y)$
  and ${z} \in [k]$ be the smallest index such that $q_{{\sigma}({z})} = q_{{\sigma}(\tilde{y})}$ (note that ${\sigma}(\tilde{y}) = y$ by definition).
  Furthermore, we have that $z > 1$ since $q_{\sigma(1)} = \max \bfq > q_y = q_{{\sigma}(z)}$.

  Let $\tau \in \mathtt{Sym}(k)$ be the permutation that swaps $z$ and $\tilde{y}$ while leaving all other elements of $[k]$ unchanged.
  Note that if $z = \tilde{y}$, then $\tau$ is the trivial permutation, i.e., the identity map on $[k]$.
  Define $\bfp := \tau(\sigma(\bfq))$, and $\bfwt := \tau(\sigma(\bfvt))$.
  Observe that $\bfp = \tau ({\sigma}(\bfq)) = {\sigma}(\bfq)$ and thus $\bfp \in \Delta^k_{\mathtt{desc}}$ as well.
  We claim that $p_{z-1} > p_z$. To see this,
  note that
  \[
    p_{z-1} = [\tau ({\sigma}(\bfq))]_{z-1}
    =[{\sigma}(\bfq)]_{\tau(z-1)}
    =[{\sigma}(\bfq)]_{z-1}
    =q_{{\sigma}(z-1)}
    > q_{{\sigma}(z)}
    =q_y
  \]
  and
  \[
    p_{z} = [\tau({\sigma}(\bfq))]_{z}
    =
    [{\sigma}(\bfq)]_{\tau(z)}
    =
    [{\sigma}(\bfq)]_{\tilde{y}}
    =
    q_{{\sigma}(\tilde{y})}
    =
    q_y.
  \]

  By Lemma~\ref{lemma:equivariance-of-conditional-risk}, we have
  \[
    \lim_t C_{\bfp}(\bfwt)=
    \lim_t C_{\tau(\sigma(\bfq))}(\tau(\sigma(\bfvt))) = \lim_t C_{\bfq}(\bfvt) = C_{\mathbf{q}}^* = C_{\tau(\sigma(\bfq))}^*
  =C_{\bfp}^*.\]
  Furthermore, we have
  $
    \max \bfvt=
    \max \sigma(\bfvt)
    =
  \max \bfwt
$
  and so
  \[
    \wt_{z} =
    [\bfwt]_{z}
    =
    [\tau({\sigma}(\bfvt))]_{z}
    =
    [{\sigma}(\bfvt)]_{\tau(z)}
    =
    \vt_{\sigma(\tau(z))}
    =
    \vt_{\sigma(\tilde{y})}
    =
    \vt_{y}
    =
    \max \bfvt
    =
    \max \bfwt.
  \]
  In summary, we have an index $z \in [k]$ where $z>1$ and a probability vector $\mathbf{p} \in \Delta^k_{\mathtt{desc}}$ such that $p_{z-1} > p_z$. Furthermore, we have a  sequence $\{\bfwt\}_t$ such that $\lim_t C_{\bfp} (\bfwt) = C_{\bfp}^*$ and
  $\wt_z = \max \bfwt$.
  This implies the desired condition in the statement of Lemma~\ref{lemma:non-ISC-property}.
  \end{proof}

\section{Omitted proof of results from Section~\ref{section:Gamma-Phi-non-calibrated-example}}\label{sec-appendix:proof-negative}

    \begin{lemma}\label{lemma:counter-example-helper}
      In the setting of Section~\ref{section:Gamma-Phi-non-calibrated-example}, we have
      $\ell = 2$ and $\bm{\alpha} = (0,0)$.
    \end{lemma}
\begin{proof}[Proof of Lemma~\ref{lemma:counter-example-helper}]
      Note that since \(\alpha_{1} =0\) already, for the ``$\bm{\alpha} = (0,0)$'' part we only need to show that \(\alpha_{2} = 0\).
      Now, to proceed, we first show that $\ell = 2$. To this end, we show that assuming $\ell \in \{1,3,\dots ,k\}$ leads to a contradiction.
      First, assume that $\ell = 1$. Then we have $\lim_t \vt_2 = \cdots = \lim_t \vt_k = -\infty$.
      Since $\gamma$ is increasing and $\phi \ge 0$,
      from Equation~\eqref{equation:counter-example-eq1} we have for any $\bfv \in \mathbb{R}^k$ that
      \[
        C_{\bfp}(\bfv) \ge
        r \gamma\Big(\sum_{j \in [k] \setminus \{1\}} \phi(v_1 - v_j)\Big)
        + (1-r) \gamma\left(\phi(v_2 - v_1) \right).
      \]
      Since $\vt_1 = 0$ for all $t$, we have
      \begin{align*}
        \lim_t C_{\bfp}(\bfvt) &\ge \lim_t
                                 r \gamma\Big(\sum_{j \in [k] \setminus \{1\}} \phi(- v_j)\Big)
                                 + (1-r) \gamma\left(\phi(v_2 ) \right)\\
                               &=r \gamma\left((k-1)\phi(+\infty) \right) + (1-r) \gamma\left(\phi(-\infty) \right) \\
                               &= r \gamma(0) + (1-r) \gamma(+\infty) \\
                               &\ge +\infty. \qquad \because \mbox{$\gamma(+\infty) = +\infty$ (Definition~\ref{definition:Gamma-conditions})}
      \end{align*}
      This is a contradiction since $C_{\bfp}(\vzero) =  \gamma\left((k-1)\phi(0)\right) < +\infty$.

      Next, we assume that $\ell \in \{3,\dots, k\}$ and derive a contradiction.
      Recall our definition that \(\bfq = (r,1-r,0,\dots, 0)\).
      Now, for a generic $\bfw \in \mathbb{R}^\ell$, recall that
      \(C_{\bfq}(\bfw)
        =
          r\calL_1(\bfw)
          +(1-r) \calL_2(\bfw)\)
      where for $y \in \{1,2\}$, we have
      \[
        \calL_y(\bfw) =
        \gamma\Big( \sum_{j\in [\ell]\setminus \{y\}}\phi(w_y- w_j) \Big).
      \]
      Let $\epsilon > 0$ and define $\bm{\beta} \in \mathbb{R}^\ell$ by
      \[
        \beta_j = \begin{cases}
                    \alpha_j &: j \ne \ell\\
                    \alpha_\ell - \epsilon &: j = \ell.
                  \end{cases}
                \]
                For $y \in \{1,2\}$, since $\beta_\ell < \alpha_\ell$ and $\beta_j = \alpha_j$ for $j \in [k] \setminus \{\ell\}$, we have
                \begin{align*}
                  &\begin{cases}
                     \beta_y - \beta_j =
                     \alpha_y - \alpha_j  &: j \ne \ell\\
                     \beta_y - \beta_\ell >
                     \alpha_y - \alpha_\ell  &: j = \ell
                   \end{cases}
                  \\
                  &\implies
                    \begin{cases}
                      \phi(\beta_y - \beta_j) =
                      \alpha_y - \alpha_j  &: j \ne \ell\\
                      \phi(\beta_y - \beta_\ell) \le
                      \phi(\alpha_y - \alpha_\ell)  &: j =\ell
                    \end{cases}
                  \\
                  &\implies
                    \calL_y(\bm{\beta}) =
                    \gamma\Big( \sum_{j\in [\ell]\setminus \{y\}}\phi(\beta_y- \beta_j) \Big)
                    \le
                    \gamma\Big( \sum_{j\in [\ell]\setminus \{y\}}\phi(\alpha_y- \alpha) \Big)
                    =
                    \calL_y(\bm{\alpha}).
                \end{align*}
                Thus, $C_{\bfq}(\bm{\alpha}) \ge C_{\bfq}(\bm{\beta})$
                and so
                $C_{\bfq}(\bm{\alpha}) \ge \lim_{\epsilon  \to \infty} C_{\bfq}(\bm{\beta})$ as well.
                By Lemma~\ref{lemma:Cp-is-Cq} and that $p_y = 0$ for $y \ge 2$, we have $\lim_t C_{\bfp}(\bfvt) = C_{\bfq}(\bm{\alpha})$.
                Now, define $\{\bftvt\}_t \subseteq \mathbb{R}^k$ by
                \[
                  \tvt_j := \begin{cases}
                              \vt_j &: j \ne \ell\\
                              -t &: j = \ell.
                            \end{cases}
                          \]
                          By construction we have $\lim_t C_{\bfp}(\bftvt) = \lim_{\epsilon \to \infty} C_{\bfq}(\bm{\beta})$ and $\{\bftvt\}_t \in \mathtt{SEQ}$.
                          Furthermore, since $\lim_t \tvt_\ell = -\infty$, we have a contradiction of the minimality of $\ell$ (Equation~\ref{equation:minimality-of-ell}).

                          Below, we can assume that $\ell =2$, where we have $\bfq = [r,1-r] \in \Delta^2_{\mathtt{desc}}$ and so
                          \begin{equation}
                            C_{\bfq}(\bm{\alpha}) =
                            r \gamma\left(\phi(- \alpha_2)\right) + (1-r) \gamma\left(\phi(\alpha_2) \right)
                            =
                            \inf_{\bfw \in \mathbb{R}^2} C_{\bfq}(\bfw).
\label{equation:counter-example-eq2}
\end{equation}
                          Recall that our goal is to show that \(\alpha_{2} = 0\). To this end, consider the function
                          \[
                            F(x) := r \gamma\left(\phi( x)\right) + (1-r) \gamma\left(\phi(-x) \right).
                          \]
                          Thus we have $C_{\bfq}(\bm{\alpha}) = \inf_{x \in \mathbb{R}} F(x)$.
                          To finish the proof, it suffices to prove that \(F\) has a unique minimum at \(0\).
 We begin by computing the derivative \(F'(x)\) of $F(x)$. Using the chain rule, we have
                          \[
                            F'(x)
                            =
                            r\gamma'(\phi(x)) \phi'(x)
                            -
                            (1-r)\gamma'(\phi(-x)) \phi'(-x).
                          \]
                          Now, $\phi'(x) = -\exp(-x)$ and
                          \[
                            \gamma'(x)
                            =
                            \begin{cases}
                              -2(x-1) & : x <1 \\
                              4(x-1) & : x \ge 1.
                            \end{cases}
                          \]
                          If $x > 0$, then $\phi(x) < 1$ and $\phi(-x) > 1$. Thus, when $x > 0$, we have
                          \begin{align*}
                            F'(x)
                            &=
                              r
                              (-2(\exp(-x) - 1))
                              (-\exp(-x))
                              -
                              (1-r)(4(\exp(x)-1)) (-\exp(x))
                            \\
                            &=
                              2r
                              (\exp(-x) - 1)
                              \exp(-x)
                              +
                              4(1-r)(\exp(x)-1) \exp(x)\\
                            & =: G_+(x).
                          \end{align*}

                          If $x \le 0$, then $\phi(x) \ge 1$ and $\phi(-x)  \le 1$. Thus, when $x \le 0$, we have
                          \begin{align*}
                            F'(x)
                            &=
                              r
                              (4(\exp(-x) - 1))
                              (-\exp(-x))
                              -
                              (1-r)(-2(\exp(x)-1)) (-\exp(x))
                            \\
                            &=
                              -4r
                              (\exp(-x) - 1)
                              \exp(-x)
                              -
                              2(1-r)(\exp(x)-1) \exp(x)\\
                            &=: G_-(x).
                          \end{align*}

                          Thus, by definition, we have
                          \[
                            F'(x) =
                            \begin{cases}
                              G_+(x) &: x > 0\\
                              G_-(x) &: x < 0\\
                              0 &: x = 0.
                            \end{cases}
                          \]
                          Finally, we prove that \(F'(x)\) vanishes only at $x = 0$
                          under the assumption that
                          $r \in [\tfrac{1}{3},\tfrac{2}{3}]$.
                          We now consider the zeros of both $G_+(x)$ and $G_-(x)$, i.e., $x \in \mathbb{R}$ where the functions vanish.
                          Clearly, both functions vanish at $x = 0$.
                          For $x \ne 0$, we compute
                          \begin{align*}
                            &0=G_+(x) =
                              2r
                              (\exp(-x) - 1)
                              \exp(-x)
                              +
                              4(1-r)(\exp(x)-1) \exp(x)
                            \\
                            &\iff
                              \tfrac{r}{2(1-r)}
                              =
                              -
                              \tfrac{  \exp(x)(\exp(x)-1) }
                              {
                              \exp(-x)
                              (\exp(-x) - 1)
                              }.
                          \end{align*}
                          Simplifying the right hand side, we have
                          \begin{align*}
                            -
                            \tfrac{  \exp(x)(\exp(x)-1) }
                            {
                            \exp(-x)
                            (\exp(-x) - 1)
                            }
                            &=
                              -\exp(2x)
                              \tfrac{  \exp(x)-1 }
                              {
                              \exp(-x) - 1
                              }
                            \\
                            &=
                              -\exp(2x)
                              \exp(x)
                              \tfrac{  1-\exp(-x) }
                              {
                              \exp(-x) - 1
                              }
                            \\
                            &=
                              \exp(3x).
                          \end{align*}

                          Thus,
                          \(0=G_+(x) \)
                          iff
                          \(
                            \tfrac{1}{3}\ln\left(\tfrac{r}{2(1-r)}\right) = x.
\)
Similarly,
                          \(                            0=G_-(x) \)
                          iff
                          \(
                            \tfrac{1}{3}\ln\left(
                              \tfrac{2r}{1-r}\right) = x.
\)
                          Thus, $G_+(x)$ has a zero on $x >0$ if and only if
                          \[
                            \tfrac{1}{3}\ln\left(
                              \tfrac{r}{2(1-r)}
                            \right)
                            > 0
                            \iff
                            \tfrac{r}{2(1-r)} > 1
                            \iff
                            r > 2/3.
                          \]
                          Similarly, $G_-(x)$ has a zero on $x < 0$ if and only if
                          \[
                            \tfrac{1}{3}\ln\left(
                              \tfrac{2r}{1-r}
                            \right)
                            < 0
                            \iff
                            \tfrac{2r}{1-r} < 1
                            \iff
                            r < 1/3.
                          \]
                          Taken together, we see that if $r \in [\frac{1}{3},\tfrac{2}{3}]$, then $F'(x)$ only vanishes at $x = 0$.
                          Moreover, \[F'(\ln(2)) = G_{+}(\ln(2)) = 8-\tfrac{17}{2}r
                            \ge
8-\tfrac{17}{2}\cdot\tfrac{2}{3} >0
                          \]
                          and likewise \(F'(-\ln(2)) = G_{-}(-\ln(2)) = \tfrac{1}{2}(1-17r) \le
\tfrac{1}{2}(1-17\tfrac{1}{3}) < 0
                          \).
                          Thus, \(F(x)\) is decreasing on \(x<0\) and increasing on \(x >0\).
                          This proves that \(F\) has a unique minimizer at \(x =0\) and
concludes the proof of Lemma~\ref{lemma:counter-example-helper}.
                        \end{proof}

                        \end{sloppypar}
\end{document}